%% file: main.tex
\documentclass{article}

\usepackage[in]{fullpage}  
\usepackage{natbib} 
\usepackage{mathpazo}

\usepackage{microtype}
\usepackage{graphicx}
\usepackage{subfigure}
\usepackage{booktabs} 

\usepackage[utf8]{inputenc} 
\usepackage[T1]{fontenc}    
\usepackage{lmodern}
\usepackage[colorlinks=true, linkcolor=blue, citecolor=blue]{hyperref}       
\usepackage{url}            
\usepackage{booktabs}       
\usepackage{amsfonts}       
\usepackage{nicefrac}       
\usepackage{microtype}      
\usepackage{amsmath}
\usepackage{amsthm}
\usepackage{amssymb}
\usepackage{graphicx}
\usepackage{algorithm,float}
\usepackage{appendix}
\usepackage{algorithmic}
\usepackage{color}
\usepackage{enumitem}
\usepackage{psfrag}

\global\long\def\real{\mathbb{R}}
\global\long\def\P{\mathbb{P}}
\global\long\def\E{\mathbb{E}}

\global\long\def\mA{\mathcal{A}}
\global\long\def\mS{\mathcal{S}}
\global\long\def\pone{\ensuremath{\mathsf{P1}}}
\global\long\def\ptwo{\ensuremath{\mathsf{P2}}}

\newcommand{\vnn}{V_{\textup{NN}}}
\newcommand{\pinn}{\pi_{\textup{NN}}}
\newcommand{\mF}{\mathcal{F}}

\theoremstyle{plain}
\newtheorem{thm}{Theorem}

\theoremstyle{plain}
\newtheorem{defn}[thm]{Definition}

\theoremstyle{plain}
\newtheorem{lem}[thm]{Lemma}

\theoremstyle{plain}
\newtheorem{property}{Property}

\theoremstyle{plain}

\theoremstyle{definition}

\theoremstyle{plain}
\newtheorem{prop}[thm]{Proposition}

\newtheorem{assumption}{Assumption}

 \title{On Reinforcement Learning for Turn-based Zero-sum\\ Markov Games}

\author{
  Devavrat Shah \\
  MIT \\
  \texttt{devavrat@mit.edu} \\
  \and
  Varun Somani \\
  Cornell University \\
  \texttt{vs472@cornell.edu} \\
  \and
  Qiaomin Xie\\
 Cornell University \\
  \texttt{qiaomin.xie@cornell.edu} \\
  \and
  Zhi Xu\\
  MIT \\
  \texttt{zhixu@mit.edu} 
}

\date{}

\begin{document}
\maketitle
\makeatother

\begin{abstract}
  We consider the problem of finding Nash equilibrium for two-player turn-based zero-sum
games. Inspired by the AlphaGo Zero (AGZ) algorithm~\citep{silver2017mastering}, we 
develop a Reinforcement Learning based approach. 
Specifically, we propose  Explore-Improve-Supervise (EIS) method that 
combines ``exploration'', ``policy improvement'' and ``supervised learning'' to find the
value function and policy associated with Nash equilibrium. We identify sufficient conditions for convergence and correctness
for such an approach. For a concrete instance of EIS where random policy is used for 
``exploration'', Monte-Carlo Tree Search is used for ``policy improvement'' and Nearest Neighbors is used for ``supervised learning'', we establish that this method finds an $\varepsilon$-approximate value function of Nash equilibrium 
in {$\widetilde{O}(\varepsilon^{-(d+4)})$} steps when the underlying state-space of the game is continuous and $d$-dimensional. This is nearly optimal as we establish 
a lower bound of {$\widetilde{\Omega}(\varepsilon^{-(d+2)})$} for any policy.

\end{abstract}

\input{ais_tex_files/intro.tex}

\input{ais_tex_files/icml_related.tex}

\input{ais_tex_files/icml_formulation.tex}

\input{ais_tex_files/generic_algorithm.tex}

\section{Properties of Modules} 
\label{subsec:properties}

In this section we formally state the desired properties of each of the three modules of EIS. With these properties, we establish convergence and correctness of EIS algorithm in Section \ref{subsec:convergence} to follow. {We remark that the properties are not made for the ease of technical analysis. Examples satisfying them shall be provided in Section \ref{sec:example}.}

\input{ais_tex_files/properties.tex}

\input{ais_tex_files/convergence.tex}

\input{ais_tex_files/imple_main.tex}

\input{ais_tex_files/conclusion.tex}

\bibliography{ais20}
\bibliographystyle{icml2020}

\appendix
\appendixpage

\input{ais_tex_files/app_pre_facts.tex}
\input{ais_tex_files/app_eis_convergence.tex}
\input{ais_tex_files/app_eis_sample.tex}
\input{ais_tex_files/app_lower_bound.tex}
\input{ais_tex_files/implementation.tex}
\input{ais_tex_files/app_mcts.tex}

\input{ais_tex_files/app_knn.tex}
\input{ais_tex_files/app_random_policy.tex}
\input{ais_tex_files/app_instance_sample.tex}
\input{app_experiment}

\end{document}

%% file: ais_tex_files/intro.tex
\section{Introduction} \label{sec:intro}

In 2016, {AlphaGo}~\citep{silver2016go} became the first program to defeat the world champion in 
the game of Go. Soon after, another program, {AlphaGo Zero}~(AGZ)~\citep{silver2017mastering}, achieved even stronger performance despite learning the game from scratch given only the rules. Starting {\em tabula rasa}, AGZ mastered the game of Go entirely through self-play using a new reinforcement learning algorithm. The same algorithm was shown to achieve superhuman performance in Chess and Shogi~\citep{silver2017chess}.

One key innovation of AGZ is to learn a policy and value function using supervised learning from samples generated via Monte-Carlo Tree Search. 
Motivated by the remarkable success of this method, in this work we study the problem of finding Nash Equilibrium for two-player turn-based zero-sum games and in particular consider a  reinforcement learning based approach. 

\medskip
\noindent{\bf Our Contributions.} The central contribution of this work is the Explore-Improve-Supervise (EIS) method for finding Nash Equilibrium for two-player turn-based  zero-sum games with \emph{continuous} state space, modeled through the framework of Markov game. It is an iterative method where in each iteration three components are intertwined carefully: ``explore'' that allows for measured exploration of the state space, ``improve'' which allows for improving the current value and policy for the state being explored, and ``supervise'' which learns the improved value and policy over the explored states so as to generalize over the entire state space.

Importantly, we identify sufficient conditions, in terms of each of the ``explore'', ``improve'' and ``supervise'' modules, under which convergence to the {value function of the} Nash equilibrium is guaranteed. In particular, we establish finite sample complexity bounds for such a generic method to find the $\varepsilon$-approximate value function of 
Nash equilibrium. See Theorem \ref{thm:eis_convergence} and Proposition \ref{prop:eis_sample_complexity} for the precise statements. 

We establish that when random sampling is used for ``explore'', Monte-Carlo-Tree-Search (MCTS) is used
for ``policy improvement'' and Nearest Neighbor is used for ``supervised learning'', the theoretical conditions
identified for convergence of EIS policy are satisfied. Using our finite sample bound for EIS
policy, and quantification of conditions as stated above, we conclude that such an instance of
EIS method find $\varepsilon$ approximate value function of  Nash equilibrium in {$\tilde{O}\big(\varepsilon^{-(d+4)})$} steps, where $d$ is the dimension of the state space
of the game (cf.\ Theorem~\ref{thm:informal_example_converge}). 
We also establish a mini-max lower bound on the number of steps required for learning $\varepsilon$-approximate value function of Nash equilibrium as {$\tilde{\Omega}\big(\varepsilon^{-(d+2)})$} for any method (cf.\ Theorem \ref{thm:lower_bound}). This
establishes near-optimality of an instance of EIS.

%% file: ais_tex_files/icml_related.tex
\medskip
\noindent{\bf Related Work.} The Markov Decision Processes (MDP) provide a canonical framework to study the single-agent setting. Its natural
extension, the Markov Games, provide a canonical framework to study multi-agent settings~\citep{littman1994markov}. In this work, we
consider an instance of it---turn-based two players or agents with zero-sum rewards. 
Analogous to learning the optimal policy in MDP setting, here we consider finding the Nash Equilibrium in the setting of Markov Games. There has
been a rich literature on existence, uniqueness as well as algorithms for finding the Nash Equilibrium. In
what follows, we describe the most relevant literature in that regard. 

To start with, in \cite{shapley1953stochastic}, for finite state and action spaces where game would terminate in a finite number of stages with positive probability, the existence of optimal stationary strategies and uniqueness of the optimal value function are established. For generic state space, the existence of Nash Equilibrium has been established for Markov Games with discounted rewards. Particularly, when the state space is a compact metric space, \cite{maitra1970stochastic,maitra1971stochastic2} and \cite{parthasarathy1973stochastic} show the uniqueness of value function and existence of optimal stationary policy. The same result has been established by \cite{kumar1981stochastic}
when the state space is complete, separable metric space.
For two-player zero-sum discounted Markov games, the Bellman operator corresponding to the Nash equilibrium is a contraction and hence, the value function is unique and  there exists a deterministic stationary optimal policy~\citep{szepesvari1996game,Hansen2013Turn_Game}. We also note that existence of Nash equilibrium for a general class of games (stochastic shortest path) is established by \cite{Patek1997SSPG}. It argues that the optimal value function is unique and can be achieved by mixed stationary strategies. 

For computing or finding optimal value function and policy associated with the Nash equilibrium, there are two settings
considered in the literature: (i) when system model is entirely known, and (ii) when model is not known but one can sample
from the underlying model. In the first setting, classical approaches from the setting of MDPs  such as value/policy iteration are adapted to find the
optimal value function or policy associated with the Nash equilibrium~\citep{Patek1997SSPG, Hansen2013Turn_Game}. 
In the second setting which is considered here, various approximate dynamic programming algorithms have been proposed~\citep{szepesvari1996game,bowling2001rational,littman2001qlearning,littman2001value,hu1998multiagent,hu2003nash,lagoudakis2002value,Perolat15Game,perolat2016API}. More recently, with the advance of deep reinforcement learning~\citep{mnih2015human,lillicrap2015continuous,schulman2015trust,schulman2017proximal,yang2019harnessing}, recent work approximates the value function/policy by deep neural networks~\citep{silver2016go,silver2017chess,silver2017mastering}.


In terms of theoretical results, there has been work establishing asymptotic convergence to  the optimal value function 
when the state space is finite. For example, Q-learning for MDP adapted to the setting of two-player zero-sum games asymptotically converges~\citep{szepesvari1996game}.
Non-asymptotic results are  available for model-based algorithms developed for Markov games with finite states, including R-max algorithm~\citep{brafman2002Rmax} and an algorithm that extends upper confidence reinforcement learning algorithm~\citep{wei2017game}. Recent work by~\cite{sidford2019solving} provides an algorithm that computes an $\varepsilon$-optimal strategy with near-optimal sample complexity for Markov games with finite states. For Markov games where the transition function can be embedded in a given feature space, the work by~\cite{jia2019featurebased} analyzes the sample complexity of a Q-learning algorithm. 
However, non-asymptotic or finite sample analysis for continuous state space without a special structure, such as that considered in this work, receives less attention in the literature.

\medskip
\noindent{\bf Comparison with Prior Work.} In this work, we develop Explore-Improve-Supervise (EIS) policy 
when the model is unknown, but one is able to sample from the underlying model. We study the convergence
and sample complexity of our approach. 
{Our goal is to provide a \emph{formal} study on the general framework of EIS. The overall framework is inspired by AlphaGo Zero and inherits similar components. However, we take an important step towards bridging the gap between sound intuitions and theoretical guarantees, which is valuable for a better understanding on applying or extending this framework with different instantiations.
{We note that EIS bears certain similarities with another AlphaGo-inspired study~\citep{shah2019mcts}. Both works follow the main idea of coupling improvements with supervised learning. However, there are major differences. \citet{shah2019mcts} focus on approximating value function in deterministic MDPs and only studies a particular instance of the modules. In contrast, we focus on a broader class of algorithms, formulating general principles and studying the guarantees. This poses different challenges and requires generic formulations on properties of the modules that are technically precise and practically implementable. }

Finally, as mentioned previously, non-asymptotic analysis for continuous state space, considered in this work, is scarce for Markov games. 
While there are some results for finite states, the bounds are not directly comparable. For example, the complexity in \cite{Perolat15Game} depends on some oracle complexities for linear programming and regression. 

For the setting with continuous state space, the sample complexity results in \cite{jia2019featurebased} {for Q-learning} rely on the assumption of linear structure of the transition kernel. The recent work by \cite{yang2019deepDQN} studies the finite-sample performance of minimax deep Q-learning for two-player zero-sum games, where the convergence rate depends on the family of neural networks. We remark that these belong to a different class of algorithms.
We also derive a fundamental mini-max lower bound on sample-complexity for any method (cf.\ Theorem \ref{thm:lower_bound}). The lower bound is interesting on its own. Moreover, it shows near optimal dependence on dimension for an instance of our EIS framework.}

\medskip
\noindent{\bf Organization.}
{The remainder of the paper is organized as follows. We formally introduce the framework of Markov Games and Nash equilibrium in Section~\ref{sec:game}. Section~\ref{sec:generic_eis} describes a generic Explore-Improve-Supervise (EIS) algorithm. The precise technical properties for the modules of EIS are then stated in Section~\ref{subsec:properties}, under which we establish our main results, convergence and sample complexity of EIS, in Section~\ref{subsec:convergence}.
Finally, a concrete instantiation is provided in Section~\ref{sec:example}, demonstrating the applicability of the generic EIS algorithm.
All the proofs are presented in Appendices.}

%% file: ais_tex_files/icml_formulation.tex
\section{Two-Player Markov Games and Nash Equilibrium}\label{sec:game}

We introduce the framework of Markov Games (MGs) (also called Stochastic Games~\citep{shapley1953stochastic}) with two players and zero-sum rewards. The goal in this setting is to learn the Nash equilibrium.

\subsection{Two-player Zero-sum Markov Game}
We consider two-player \emph{turn-based} Markov games like Go and Chess, where players take turns to make decisions. 
We denote the two players as $\pone$ and $\ptwo$. Formally, a Markov game can be expressed as a 
tuple $ (\mS_1, \mS_2,\mA_1, \mA_2, r, P, \gamma), $ where $\mS_1$ and $\mS_2$ are the set of states controlled by $\pone$ and $\ptwo$ respectively, $\mA_{1} $ and $\mA_2$ are the set of actions players can take,  $r$ represents reward function, $P$ represents transition kernel and $\gamma \in [0,1)$
is the discount factor. Specifically, for $i=1,2,$ let $\mA_{i}(s)$ be the set of feasible actions for player $i$ in a given state $s\in\mS_{i}$.
We assume that $\mS_{1}\cap\mS_{2}=\emptyset$\footnote{This 
assumption is typical for turn-based games. In particular, one can incorporate the ``turn'' of 
the player as part of the state, and thus \pone's state space (i.e. when it is 
player 1's turn) is disjoint from that of \ptwo's turn. }. 

Let $\mS=\mS_{1}\cup\mS_{2}$.
For each state $s\in\mS,$ let $I(s)\in\{1,2\}$ indicate the current player to play. At state $s$, upon taking action $a$ by the corresponding player $I(s)$, player $i \in \{1,2\}$ receives a reward $r^{i}(s,a)$. 
In zero-sum games, $r^{1}(s,a)=-r^{2}(s,a)$. Without loss of generality, we let $\pone$ be our reference and use the notation $r(s,a)\triangleq r^1(s,a)$ for the definitions of value functions. 
 
Let $P(s,a)$ be the distribution of the new state after playing action $a$, in state $s$, by 
player $I(s)$. In this paper, we focus on the setting where the state transitions are deterministic. This means that $P(s,a)$ is supported on a single state, $s\circ a$, where $s\circ a$ denotes the state after taking action $a$ at state $s$.

For each $i\in\{1,2\}$, let $\pi_{i}$ be the policy for player $i$, where $\pi_{i}(\cdot|s)$ is 
a probability distribution over $\mA_{i}(s)$. Denote by $\Pi_{i}$ the set of all stationary policies of 
player $i$, and let $\Pi=\Pi_{1}\times\Pi_{2}$ be the set of all polices for the game. 
A two-player zero-sum game can be seen as player $\pone$ aiming to maximize the accumulated discounted reward while $\ptwo$ attempting to minimize it. 
The value function and Q function for a zero-sum Markov game can be defined in a manner analogous to the MDP setting:
\begin{align*}
V_{\pi_{1},\pi_{2}}(s) & =\E_{a_{t},s_{t+1},a_{t+1},\ldots} \big[ \sum_{k=0}^{\infty}\gamma^{k}r(s_{t+k},a_{t+k})|s_{t}=s \big],\\
Q_{\pi_{1},\pi_{2}}(s,a)  &=\E_{s_{t+1},a_{t+1},\ldots} \big[ \sum_{k=0}^{\infty}\gamma^{k}r(s_{t+k},a_{t+k})|s_{t}=s,a_{t}=a \big],
\end{align*}
where $a_{l}\sim\pi_{I(s_{l})}(\cdot|s_{l})$ and $ s_{l+1}\sim P(s_{l},a_{l})$. That is, $V_{\pi_{1},\pi_{2}}(s)$
is the expected total discounted reward for $\pone$ if the
game starts from state $s$, players $\pone$ and $\ptwo$ use the policies $\pi_{1}$ and $\pi_{2}$ respectively. The interpretation for Q-value is similar.  

To simplify the notation, we assume that $\mathcal{A}_1=\mathcal{A}_2\triangleq \mathcal{A}$, where $\mathcal{A}$ is a finite set. We consider $\mathcal{S}$ to be a compact subset of $\mathbb{R}^d$ ($d \geq 1$). The rewards $r(s,a)$ are independent random variables taking value in $[-R_{\max}, R_{\max}]$ for some $R_{\max} > 0$. Define $V_{\max}\triangleq{R_{\max}/{(1-\gamma)}}.$ It follows that absolute value of value function and Q function for any policy is bounded by $V_{\max}.$

\medskip
{
\noindent {\em Regarding Deterministic Transitions.} Let us clarify this assumption. In fact, our approach and main results of EIS framework (i.e., Sections \ref{subsec:properties} and \ref{subsec:convergence}) apply to general non-deterministic cases as well. However, the example in Section \ref{sec:example} considers deterministic cases. In particular, the improvement module is instantiated by a variant of Monte Carlo Tree Search, where a clean non-asymptotical analysis has been only established for the deterministic case \citep{shah2019mcts}. To facilitate a coherent exposition, we focus on deterministic cases here. Indeed, many games, such as Go and Chess, are deterministic. Additionally, note that one could instantiate our EIS framework with other methods for the non-deterministic cases---for instance, by adapting the sparse sampling oracle \citep{kearns2002sparse} as the improvement module---to obtain a similar analysis. {As a proof of concept, we provide empirical results in Appendix \ref{app:empirical} on a  non-deterministic game with the sparse sampling oracle.} }

\subsection{Nash Equilibrium}

\begin{algorithm*}[h]
  \caption{The Generic EIS Algorithm}
  \label{alg:eis}
\begin{algorithmic}[1]
  \STATE {\bfseries Initialization:} a supervised learning model $f_0(s)=(V_0(s),\pi_0(\cdot|s))$ for every $s\in\mathcal{S}$.
  \FOR {$l = 1,2 ,\dots, L$}
  \STATE initialize $s_1$ to be a state (e.g., the starting state of the game).
  \STATE \texttt{/*}~~\texttt{data generation via improvement and exploration}~~\texttt{*/}
  \FOR {$i=1,2,\dots, n_l$}
  \STATE query the \emph{improvement module}, which takes as inputs the current model $f_{l-1}$, and outputs estimates $\hat{V}(s_i)$ for the optimal value $V^*(s_i)$ and $\hat{\pi}(\cdot|s_i)$ for the optimal policy ${\pi}^*(\cdot|s_i)$:
  \vspace{-0.05in}
  \begin{equation}
      \big(\hat{V}(s_i),\:\: \hat{\pi}(\cdot|s_i)\big) = \textrm{Improvement Module}(f_{{l-1}},s_i)
      \vspace{-0.05in}
  \end{equation}
  \STATE query the \emph{exploration module} for the next state $s_{i+1}$:
  \vspace{-0.05in}
  \begin{equation}
      s_{i+1} = \textrm{Exploration Module}(s_i)
      \vspace{-0.1in}
  \end{equation}
  \ENDFOR
 
  \STATE \texttt{/*}~~\texttt{model update via supervised learning}~~\texttt{*/}
  \STATE query the \emph{supervised learning module} with the collected data $\mathcal{D}^{(l)}=\{(s_i,\hat{V}(s_i),\hat{\pi}(\cdot|s_i)\}_{i=1}^{n_l}$ and obtained an updated model $f_{l}(s)$ for every $s\in\mathcal{S}$.
  \vspace{-0.05in}
  \begin{equation}
      f_{l} = \textrm{Supervised Learning Module}\big(\{(s_i,\hat{V}(s_i),\hat{\pi}(\cdot|s_i)\}_{i=1}^{n_l}\big)
      \vspace{-0.1in}
  \end{equation}

  \ENDFOR
  \STATE {\bfseries Output:} final model $f_{L}$.
\end{algorithmic}
\end{algorithm*}

\begin{defn}[Optimal Counter Policy]
    Given a policy $\pi_{2}\in\Pi_{2}$, policy $\pi_{1} \in \Pi_1$
	for $\pone$ is said to be an optimal counter-policy against $\pi_{2}$,
	if and only if for every $s \in \mS$, we have
	$
	V_{\pi_{1},\pi_{2}}(s) \geq V_{\pi_{1}',\pi_{2}}(s) ,\forall\pi_{1}'\in\Pi_{1}.
	$
	Similarly, a policy $\pi_{2}\in\Pi_2$ for $\ptwo$ is said to be an optimal
	counter-policy against a given policy $\pi_{1}\in\Pi_1$ for $\pone$, if and only if for every $s \in \mS$
	$
	V_{\pi_{1},\pi_{2}}\leq V_{\pi_{1},\pi_{2}'},\forall\pi_{2}'\in\Pi_{2}.
	$
\end{defn}

In a two-player zero-sum game, it has been shown that the pairs of optimal policies coincides with the Nash equilibrium of this game~\citep{maitra1970stochastic,parthasarathy1973stochastic,kumar1981stochastic}.
In particular, a pair of policies $(\pi_{1}^{*},\pi_{2}^{*})$
is called an equilibrium solution of the game, if $\pi_{1}^{*}$ is
an optimal counter policy against $\pi_{2}^{*}$ and $\pi_{2}^{*}$
is an optimal counter policy against $\pi_{1}^{*}$. The value function of the optimal policy, 
$V_{\pi_1^*,\pi_2^*},$ is the \emph{unique} fixed point of a $\gamma$-contraction operator.  
In the sequel, we will simply refer to the strategy $\pi^*=(\pi_1^*,\pi_2^*)$ as the optimal policy. Finally, we use the concise notation $V^*$ and $Q^*$ to denote the optimal value function and the optimal Q-value, respectively, i.e., $V^*(s)=V_{\pi^*_1,\pi^*_2}(s)$ and $Q^*(s,a)=Q_{\pi^*_1,\pi^*_2}(s,a)$.

%% file: ais_tex_files/generic_algorithm.tex
\section{EIS: Explore-Improve-Supervise}
\label{sec:generic_eis}

We describe Explore-Improve-Supervise (EIS) algorithm for learning the optimal value function $V^*$ and optimal policy $\pi^*$. The algorithm consists of three separate, but intertwined modules: exploration, improvement and supervised learning. Below is a brief summary of these modules. The precise, formal description of properties desired from these modules is stated in Section~\ref{subsec:properties}, which will lead to convergence and correctness of the EIS algorithm as stated in Theorem~\ref{thm:eis_convergence}. Section \ref{sec:example} provides a concrete example of modules of EIS satisfying properties stated in Section \ref{subsec:properties}.

{\bf Exploration Module.} To extract meaningful information for the entire game, sufficient exploration is required so that enough \emph{representative} states will be visited. This is commonly achieved by an appropriate exploration policy, such as $\epsilon$-greedy policy and Boltzmann policy. We require the existence of an exploration module guaranteeing sufficient exploration.

{\bf Improvement Module.} For the overall learning to make any progress, the improvement module improves the existing estimates of the optimal solution. In particular, given the current estimates $\hat{V}$ for $V^*$ and $\hat{\pi}$ for $\pi^*$, for a state $s$, a query of the improvement module produces better estimates $\hat{V}'(s)$ and $\hat{\pi}'(\cdot|s)$ that are \emph{closer} to the optimal $V^*(s)$ and $\pi^*(\cdot|s)$.

{\bf Supervised Learning Module.} The previous two modules can be collectively viewed as a data generation process: the exploration module samples sufficient representative states, while a query of the improvement module provides improved estimates for the optimal value and policy. With these as training data, supervised learning module would learn and generalize the improvement of the training data to the entire state space. Subsequently, the trained supervised learning module produces better estimates for $V^*$ and $\pi^*$. 

Combining together, the three modules naturally lead to the following iterative algorithm whose pseudo-code is provided in Algorithm \ref{alg:eis}. Initially, the algorithm starts with an arbitrary model for value function and
policy. In each iteration $l \geq 1$, it performs two steps:

{\em Step 1. Data Generation.} Given current model $f_{l-1}=(V_{l-1},\pi_{l-1})$: for current state $s$, 
query the improvement module to obtain better estimates $\hat{V}(s)$  and $\hat{\pi}(\cdot|s)$ than 
the current estimates $f_{l-1}(s)$; and then query the exploration module to arrive at the next state $s'$; 
repeat the above process to obtain training data of $n$ samples, $\{(s_i,\hat{V}(s_i),\hat{\pi}(\cdot|s_i))\}_{i=1}^n$. 
    
{\em Step 2. Supervised Learning.} Given the improved estimates  $\{(s_i,\hat{V}(s_i),\hat{\pi}(\cdot|s_i))\}_{i=1}^n$, use the supervised learning module to build a new model $f_{l} = (V_l, \pi_l)$.



Intuitively, the iterative algorithm keeps improving our estimation after each iteration, and eventually converges to optimal solutions. The focus of this paper is to \emph{formally} understand under what conditions on each of 
the exploration, improvement and supervised learning module does the algorithm work.
Of course, proof is in the puddling---we provide examples of existence of such modules in Section \ref{sec:example}.

%% file: ais_tex_files/properties.tex
\subsection{Improvement Module}
This module improves both value function and policy. The value
function is real-valued, whereas policy for each given state can be viewed as a probability distribution over
all possible actions. This requires a careful choice of metric for quantifying improvement. Let $\hat{V}(s)$ and $\hat{\pi}(\cdot|s)$ be the estimates output by the improvement module in the $l$-th iteration of EIS. Improvement of
value function means $|\hat{V}(s)-V^*(s)| < |V_l(s)-V^*(s)|$. Improvement for
policy is measured by the KL divergence between $\hat{\pi}(\cdot|s)$ and $\pi^*(\cdot|s)$. Here some care is needed as KL divergence would become infinite if supports of the distributions mismatch.

Note that the optimal policy $\pi^*$ only assigns positive probability to the optimal actions. On the other hand,  there is no guarantee that $\hat{\pi}(\cdot|s)$ always has a full support on $\mathcal{A}.$ To overcome these challenges, we instead measure the KL divergence with an alternative "optimal policy" that guarantees a full support on $\mathcal{A}$. This naturally leads to the optimal Boltzmann policy: given a temperature $\tau > 0$, the optimal Boltzmann policy is given by 
\begin{equation}
    P^*_\tau(a|s) = \frac{e^{Q^*(s,a)/\tau}}{\sum_{a'\in\mathcal{A}}e^{Q^*(s,a')/\tau}}, \quad \text{for } a\in\mathcal{A}. \label{eq:def_opt_Boltzmann}
\end{equation}
If $I(s)$ is player $\ptwo$, use $-Q^*(s,a)$ instead in the above equation to construct the Boltzmann policy (Recall that player $\pone$ is set to be our reference in Section \ref{sec:game}). By definition, $D_{\textup{KL}}\big(\hat{\pi}(\cdot|s)|| P^*_\tau(\cdot|s)\big)$ is guaranteed to be finite for any estimate $\hat{\pi}(\cdot|s)$. Furthermore, $P^*_\tau$ converges to $\pi^*$ as $\tau\rightarrow 0$. Therefore, 
we could use the KL divergence $D_{\textup{KL}}\big(\hat{\pi}(\cdot|s)|| P^*_\tau(\cdot|s)\big)$ with a small enough $\tau$ to measure the improvement of the estimates. 

Finally, it makes sense to take into account the number of samples (i.e., observed state transitions) required by the module to improve the policy and value function. 
We now formally lay down the following property for the improvement module.
\begin{property}
\label{property:improvement}
{\bf (Improvement Property)}
Suppose the current model $f(s)=(V(s),\pi(\cdot|s))$ (potentially random) has estimation errors $\varepsilon_{0,v}>0$ and $\varepsilon_{0,p}>0$ for the value and policy estimates, respectively, i.e., 
\begin{align*}
    \E\Big[||V- V^*||_\infty\Big]&\leq \varepsilon_{0,v},  \\
     \E\Big[D_{\textup{KL}}\big(\pi(\cdot|s)||P^*_\tau(\cdot|s)\big)\Big] &\leq \varepsilon_{0,p},\:\forall\:s\in\mathcal{S},
\end{align*}
where the expectations are taken with respect to the randomness of the model $f=(V,\pi)$.

Fix any state $s\in\mathcal{S}$, and query the improvement module on $s$ via $ \big(\hat{V}(s), \hat{\pi}(\cdot|s)\big) = \textrm{Improvement Module}(f,s)$.
Let the temperature be $\tau>0$, and improvement factors be $0<\zeta_{v}<1$ and $0<\zeta_{p}<1$. Then, there exists a function $\kappa(\tau,\varepsilon_{0,v},\varepsilon_{0,p},\zeta_v,\zeta_p)$ such that
if $\kappa(\tau,\varepsilon_{0,v},\varepsilon_{0,p},\zeta_v,\zeta_p)$ number of samples are used, the new estimates satisfy that
\begin{align*}
    \mathbb{E}\Big[\big|\hat{V}(s)-V^*(s)\big|\Big]&\leq\zeta_v\cdot\varepsilon_{0,v},\\
    \mathbb{E}\Big[D_{\textup{KL}}\big(\hat{\pi}(\cdot|s)||P^*_\tau(\cdot|s)\big)\Big]&\leq \zeta_p\cdot\varepsilon_{0,p},
\end{align*}
where the expectations are with respect to the randomness in the model $f$ and the improvement module. 
\end{property}
Property \ref{property:improvement} allows for a randomized improvement module, but requires that on average, the errors for the value and policy estimates should strictly shrink.

\subsection{Supervised Learning Module}
To direct the model update in an improving manner, the supervised learning step (line 10 of Algorithm \ref{alg:eis}) should be able to learn from the training data, $\hat{V}$ and $\hat{\pi}$, and generalize to unseen states by preserving the same order of error as the training data. Generically speaking, generalization would require two conditions: (1) sufficiently many training data that are ``representative'' of the underlying state space; (2) the model itself is expressive enough to capture the characteristics of the function that is desired to be learned. 

Before specifying the generalization property, let us provide a few remarks on the above conditions. Condition (1) is typically ensured by using an effective exploration module. Recall that the state space is continuous. The exploration module should be capable of navigating the space until sufficiently many different states are visited. Intuitively, these \emph{finite} states should properly cover the entire space, i.e., they are representative of the entire space so that learning from these states provide enough information for other states. Formally, this means that given the current estimation errors $\varepsilon_{1,v}$ and $\varepsilon_{1,p}$ for the optimal value and policy, there exists a sufficiently large set of $N(\varepsilon_{1,v},\varepsilon_{1,p})$ training states, 
such that supervised learning applied to those training data would generalize  to the entire state space with the same order of accuracy. The precise definition of representative states 
may depend on the particular supervised learning algorithm.

Regarding condition (2), generalization performance of traditional models has been well studied in classical statistical learning theory. More recently, deep neural networks exhibit superior empirical generalization ability,  although a complete rigorous proof seems beyond the reach of existing techniques. Our goal is to seek general principle underlying the supervised learning step and as such, we do not limit ourselves to specific models---the learning model could be a parametric model that learns via minimizing empirical squared loss and cross-entropy loss, or it could be a non-parametric model such as nearest neighbors regression. 
With the above conditions in mind, we state the following general property for the supervised learning module:
\begin{property}
\label{property:sl}
{\bf (Generalization Property) }
Let temperature $\tau>0$, estimation errors $\varepsilon_{1,v}>0$ and $\varepsilon_{1,p}>0$ be given. There exists at least one set of finite states, denoted by $\mathcal{S}(\tau,\varepsilon_{1,v},\varepsilon_{1,p})$, with size $N_\mS(\tau,\varepsilon_{1,v},\varepsilon_{1,p})$, so that the following {\bf generalization bound} holds:

Suppose that a training dataset $\big\{\big(s_{i},\hat{V}(s_{i}),\hat{\pi}(\cdot|s_{i})\big)\big\}_{i=1}^{n}$  satisfies  $\mathcal{S}(\tau,\varepsilon_{1,v},\varepsilon_{1,p})\subset\{s_i\}_{i=1}^n$ and the following error guarantees: 
\begin{align*}
\E\Big[|\hat{V}(s_{i})-V^{*}(s_{i})|\Big] & \leq\varepsilon_{1,v},\\
\E\Big[D_{\textup{KL}}\big(\hat{\pi}(\cdot|s_{i})\Vert P^{*}_\tau(\cdot|s_{i})\big)\Big] 
&\leq\varepsilon_{1,p},\:\forall\:i\in[n],
\end{align*}
where the expectation is taken with respect to the randomness of the value $\hat{V}(s_{i})$ and $\hat{\pi}(\cdot|s_{i})$. Then, 
there exist non-negative universal constants $c_{p}$ and
$c_{v}$ such that after querying the supervised learning module, i.e., $(V,\pi)=${ Supervised Learning Module}$(\{\big(s_{i},\hat{V}(s_{i}),\hat{\pi}(\cdot|s_{i})\big)\}_{i=1}^{n})$, $(V,\pi)$ satisfy
\begin{align*}
\E\Big[\big\Vert V-V^{*}\big\Vert_{\infty}\Big] 
&\leq  c_v \cdot \varepsilon_{1,v};\\
\quad \E\Big[D_{\textup{KL}}\big(\pi(\cdot|s)\Vert P^{*}_\tau(\cdot|s)\big)\Big] 
&\leq c_p \cdot \varepsilon_{1,p}, \quad \forall s \in \mS.
\end{align*}
\end{property}

\subsection{Exploration Module}
With the above development, it is now straightforward to identify the desired property of the exploration module. In particular, as part of the data generation step, it should be capable of exploring the space so that a set of representative states $\mathcal{S}(\tau,\varepsilon_{1,v},\varepsilon_{1,p})$ are visited. Consequently, the supervised learning module can then leverage the training data to generalize. Formally, let $\mathcal{E}$ be the set of all possible representative sets that satisfy the Generalization Property:
\begin{align*}
\vspace{-0.15in}
\mathcal{E}(\tau,\varepsilon_{1,v},\varepsilon_{1,p})=\Big\{\mathcal{S}(\tau,\varepsilon_{1,v},\varepsilon_{1,p})\subset \mS\::\:
\textrm{
Property \ref{property:sl} is satisfied with }\mathcal{S}(\tau,\varepsilon_{1,v},\varepsilon_{1,p}).\Big\}.    
\vspace{-0.15in}
\end{align*}
Denote by $\mathcal{T}(t)\triangleq\{s_i\}_{i=1}^t$ the set of states explored by querying the exploration module up to time $t$, with $s_1$ being the initial state and $s_{i+1}=\textrm{Exploration Module}(s_i)$ (cf. line 7 of Algorithm \ref{alg:eis}). We now state the exploration property, which stipulates that starting at an arbitrary state $s$, the explored states should contain one of the representative sets in $\mathcal{E}$, within a \emph{finite} number of steps.

\begin{property}
\label{property:exploration}
{\bf (Exploration Property) }
Given the temperature $\tau>0$, and estimation errors $\varepsilon_{1,v}>0$ and $\varepsilon_{1,p}>0$ for the value and policy, define 
\begin{align*}
\vspace{-0.15in}
T &(\tau,\varepsilon_{1,v},\varepsilon_{1,p},s) \triangleq \min\Big\{\:t\geq 1\::\:s_1=s; \exists\: \hat{\mathcal{S}}\in \mathcal{E}(\tau,\varepsilon_{1,v},\varepsilon_{1,p})\textrm{ such that }\hat{\mathcal{S}}\subset\mathcal{T}(t)\Big\}.
\vspace{-0.15in}
\end{align*}
Then, the exploration module satisfies that $\forall s\in\mS$, $$\mathbb{E}\Big[T (\tau,\varepsilon_{1,v},\varepsilon_{1,p},s)\Big]<B(\tau,\varepsilon_{1,v},\varepsilon_{1,p}),$$ for some  $B(\tau,\varepsilon_{1,v},\varepsilon_{1,p})<\infty$ independent of $s$. The above expectation is taken with respect to the randomness in the exploration module and the environment (i.e., state transitions).
\end{property}
In the sequel, when the context is clear or the initial state does not matter, we usually drop the dependence in $s$ to simplify the notation, i.e., $T(\tau,\varepsilon_{1,v},\varepsilon_{1,p})$.

%% file: ais_tex_files/convergence.tex
\section{Main Results: Convergence Guarantees and Sample Complexity}
\label{subsec:convergence}

\subsection{Convergence Guarantees}
As the main result of this paper, we establish convergence of the EIS algorithm under the three desired properties given in Section~\ref{subsec:properties}, and quantify the corresponding finite sample complexity. 
We also provide an algorithm-independent minimax lower bound; in Section~\ref{sec:example} we introduce  an instance of EIS that essentially matches this lower bound. 

\begin{thm} \label{thm:eis_convergence}
Given a small enough $\tau>0$, let Properties \ref{property:improvement}, \ref{property:sl} and \ref{property:exploration} hold. 
Let $C_{0,v}= \| V_{0} - V^* \|_\infty$ and $C_{0,p}=\sup_{s\in \mS}D_{\textup{KL}}\big(\pi_0(\cdot|s) \Vert P^*_\tau(\cdot|s)\big)$ be initialization errors. Then for a given $\rho\in (0,1),$ with appropriate parameters for Algorithm~\ref{alg:eis}, the output $f_{L}=(V_L,\pi_L)$ after $L$-th iteration satisfies
\begin{align}
    \E\Big[\big\Vert V_{L} -V^* \big\Vert_\infty \Big] &\leq C_{0,v}\rho^L,\label{eq:thm_eis_convergence_value}\\
    \E\Big[D_{\textup{KL}}\big(\pi_{L}(\cdot|s) \,\Vert\, P_\tau^{*}(\cdot|s)\big)\Big] &\leq  C_{0,p}\rho^L, \quad \forall s \in \mS.  \label{eq:thm_eis_convergence_policy}
\end{align}
\end{thm}
Theorem~\ref{thm:eis_convergence} implies that the model $(V_L,\pi_L)$ learned by the generic EIS algorithm converges to the optimal value function $V^*$ and the optimal Boltzmann policy $P^*_\tau$ exponentially with respect to the number of iterations. In particular, after $$L=\bigg\lceil\frac{\log\frac{\varepsilon}{\max\{C_{0,v},C_{0,p}\}}}{\log\rho}\bigg\rceil=\Theta\Big(\log \frac{1}{\varepsilon}\Big)$$ iterations, we can obtain estimates for both $V^*$ and $P_\tau^*$ that are within $\varepsilon$ estimation errors. We note that with a sufficiently small temperature, $P^*_\tau$ is close to the optimal policy $\pi^*$. 
Therefore, the model $f_{L}=(V_{L},\pi_{L})$ can be close to $(V^*,\pi^*)$ for a large $L$.
 
 \subsection{Sample Complexity}
We can also characterize the sample complexity of the EIS algorithm. Recall that the sample complexity is defined as the total number of state transitions required for the algorithm to learn $\epsilon$-approximate value/policy function. The sample complexity of Algorithm \ref{alg:eis} comes from two parts: the improvement module and the exploration module. Recall that the improvement module requires $\kappa(\tau,\varepsilon_{0,v},\varepsilon_{0,p},\zeta_v,\zeta_p)$ samples for each call (cf.\ Property~\ref{property:improvement}). The sample complexity of exploration module is proportional to  $T(\tau,\varepsilon_{1,v},\varepsilon_{1,p})$, which satisfies 
$\E[T(\tau,\varepsilon_{1,v},\varepsilon_{1,p})]\leq B(\tau,\varepsilon_{1,v},\varepsilon_{1,p})$ (cf.\ Property~\ref{property:exploration}). The following proposition bounds the sample complexity in terms of the above relevant quantities.
\begin{prop}
\label{prop:eis_sample_complexity}
Consider the setting of Theorem \ref{thm:eis_convergence}. Then, with probability at least $1-\delta$, the convergence result (i.e., Eqs~(\ref{eq:thm_eis_convergence_value}) and (\ref{eq:thm_eis_convergence_policy})) is achieved with sample complexity $M$ such that 
\begin{align*}
    M\leq \sum_{l=1}^L\Bigg[\kappa\Big(\tau,C_{0,v}\rho^{l-1},C_{0,p}\rho^{l-1},\frac{\rho}{c_v},\frac{\rho}{c_p}\Big) \times  B\Big(\tau,\frac{C_{0,v}\rho^{l}}{c_v},\frac{C_{0,p}\rho^{l}}{c_p}\Big)\cdot e\cdot\log \frac{L}{\delta}\Bigg] .
\end{align*}
\end{prop}
In Section \ref{sec:example}, we provide a concrete instance of EIS that finds $\varepsilon$-approximate value function and policy of Nash equilibrium with $\widetilde{O}(\varepsilon^{-(d+4)})$ transitions.

\subsection{A Generic Lower Bound}
To understand how good the above sample complexity upper bound is, we establish a \emph{lower bound }for any algorithm under any sampling policy. In particular, we leverage the the minimax lower bound for the problem of non-parametric regression \citep{tsybakov2009nonparm,stone1982optimal} to establish the lower bound, as stated in the following theorem. 
\begin{thm} \label{thm:lower_bound}
    Given an algorithm, let $V_T$ be the estimate of $V^*$ after $T$ samples of transitions for the given Markov game. Then, for each $\delta\in(0,1)$, there exists a two-player zero-sum Markov game such that in order to achieve 
	$
	\P\Big[\big\Vert \hat{V}_T-V^{*}\big\Vert_{\infty}<\varepsilon\Big]\ge 1-\delta,
	$
    it must be that
	\[
	T\ge C'd\cdot\varepsilon^{-(d+2)}\cdot\log(\varepsilon^{-1}),
	\]
	where $C'>0$ is an algorithm-independent constant.
\end{thm}

%% file: ais_tex_files/imple_main.tex
\section{Implementation: A Concrete Instantiation of the Key Modules}\label{sec:example}

In this section, we demonstrate the applicability of the generic EIS algorithm by giving a concrete instantiation. Specifically, we will use a variant of Monte Carlo Tree Search (MCTS) as the improvement module, nearest neighbor regression as the supervised learning module,  and random sampling as the exploration module. We prove that all properties in Section~\ref{subsec:properties} are satisfied. This shows that these properties are reasonable and hence gives a strong support for the generic recipe developed in this paper. Due to space limit, we provide high-level discussions here with \emph{informal} technical results, and defer precise statements to Appendix~\ref{appendix:example}.

\medskip
{\bf Improvement Module: MCTS.}
Recall that the improvement module should be able to provide improved estimates for the value and policy functions, at the queried state $s$. Since both the value and policy are related to the $Q$ function, one  approach for estimate improvement is to first obtain better estimates $\hat{Q}$ for  $Q^*$ and then construct the improved estimates of value and policy from $\hat{Q}$. We will take this approach in this example and use MCTS to obtain the estimates of $Q^*$  (see Algorithm \ref{alg:mcts} in Appendix \ref{appendix:example}). We assume the existence of a generative model (i.e., a simulator). The following theorem states the property of this specific improvement module, which directly implies the desired improvement property, i.e., Property~\ref{property:improvement}.

\begin{thm}[Informal Statement, Theorem \ref{thm:mcts_improvement}, Appendix \ref{appendix:example}]
Suppose that the state transitions are deterministic. Given the current model $f=(V,\pi)$ such that  
the value model $V$ satisfies $\E\big[||V- V^*||_\infty\big]\leq \varepsilon_{0,v}.$
Then, with appropriately chosen parameters for MCTS, for each query state $s_0\in\mathcal{S}$, the output $ \big(\hat{V}(s_0), \hat{\pi}(\cdot|s_0)\big) = \textrm{MCTS}(f,s_0)$ satisfies 
 \begin{align*}
        \mathbb{E}\Big[\big|\hat{V}(s_0)-V^*(s_0)\big|\Big]\leq\zeta_v\cdot\varepsilon_{0,v},\\
        \mathbb{E}\Big[D_{\textup{KL}}\big(\hat{\pi}(\cdot|s_0)||P^*_\tau(\cdot|s_0)\big)\Big]\leq \zeta_p\cdot\varepsilon_{0,v}.
    \end{align*}
The above is achieved with a sample complexity of 
\begin{equation*}
O\Big(\big({\tau\cdot\min\{\zeta_v,\zeta_p\}\cdot\varepsilon_{0,v}}\big)^{-2}\cdot {\log({\tau\cdot\min\{\zeta_v,\zeta_p\}})}\cdot({\log \gamma})^{-1}\Big).
\end{equation*}

\end{thm}
 
 \medskip
{\bf Supervised Learning Module: Nearest Neighbor Regression.}
We employ a nearest neighbor  algorithm to learn the optimal value function and policy.  Intuitively, {suppose that the optimal value function and the Boltzmann policy is Lipschitz in the state space, then} this algorithm will generalize if there are sufficiently many (say $K$)
training data points around each state in the state space $\mS$. Quantitatively, consider covering $\mS$  with balls of diameter $h>0$. We call the training data \emph{$(h,K)$}-\emph{representative} if  each  covering ball has at least $K$ training data. Here, $h$ and $K$ would depend on the temperature $\tau$ and estimation errors of the training data.  

\begin{prop}[Informal Statement, Proposition \ref{prop:knn}, Appendix \ref{appendix:example}]
Under appropriate regularity conditions, if the training data is representative with respect to appropriate chosen $h>0$ and $K>0$, the nearest neighbor supervised learning satisfies Property~\ref{property:sl}. In particular, given training data with estimation errors $\varepsilon_v$ and $\varepsilon_p$,  we have 
\begin{align*}
\E\Big[\left\Vert \vnn-V^{*}\right\Vert _{\infty}\Big]  &\le 4 \cdot \varepsilon_{v},\\ 
\E\Big[D_{\textup{KL}}\big(\pinn(\cdot|s)\,\Vert\,P^{*}_\tau(\cdot|s)\big)\Big]  &\le c\cdot\varepsilon_{p},\quad \forall s \in \mS,
\end{align*}
where the constant $c>0$ is independent of $n$, the size of training data. 
\end{prop}
{As discussed in Appendix~\ref{appendix:example}, the representative number of data points for training required in
the above for generalization depends on the property of the state-space. For example, if state space is the unit ball in $d$
dimension, for generalization error scaling with $\varepsilon$ we require representative data points scaling as 
$\varepsilon^{-(2+d)}$. }

\medskip
{\bf Exploration Module: Random Sampling Policy.}
In the above supervised learning module, the sampled states for nearest neighbor regression should be $(h,K)$-representative. In other words, to satisfy the exploration property, the exploration module must  visit a set of $(h,K)$-representative states within a finite expected number of steps. 
We show that a uniformly random sampling policy achieves this. Let $N(h)$
be the $h/2$-covering number of the compact state space. 

\begin{prop}[Informal Statement, Proposition \ref{prop:random_policy}, Appendix \ref{appendix:example}]
Under appropriate regularity conditions, with $h, K$ chosen as per desired the estimation errors, 
$\varepsilon_v$ and $\varepsilon_p$, for the value and policy,   the expected number of steps to obtain a 
set of $(h,K)$-representative states under the random sampling policy is upper bounded by 
$$B(\tau,\varepsilon_v,\varepsilon_p)=O\big(KN(h) \log N(h) \big).$$
\end{prop}

\medskip
{\bf Convergence Guarantees and Sample Complexity of the Instance.} For this instance of EIS,
we have shown that each module satisfies the desired properties. Therefore, the convergence result stated in Theorem~\ref{thm:eis_convergence} holds for this specific instance. Below we make this result explicit, providing concrete bounds on the estimation errors and sample complexity. In the following, the $c'$s denote appropriate constants. Please refer to Appendix \ref{appendix:example} for details.

\begin{thm}[Informal Statement, Theorem \ref{thm:instance_of_eis}, Appendix \ref{appendix:example}]
\label{thm:informal_example_converge}
For a given $\rho\in (0,1)$, there exist appropriately chosen parameters for this instance such that:
\begin{enumerate}
    \item The output $f_{L}=(V_L,\pi_L)$ at the end of $L$-th iteration satisfies \begin{align*}
    \E\Big[\big\Vert V_{L} -V^* \big\Vert_\infty \Big] &\leq V_{\max}\rho^L\\ 
    \E\Big[D_{\textup{KL}}\big(\pi_{L}(\cdot|s) \,\Vert\, P_\tau^{*}(\cdot|s)\big)\Big] & \leq  \frac{cV_{\max}}{4} \rho^L, \quad \forall s \in \mS.
    \end{align*}
    
    \item {With probability at least $1-\delta,$ the above result is achieved with sample complexity of }
    \begin{align*}
    \hspace{-0.2in}
        \sum_{l=1}^{L}c' \log \frac{1}{\tau\rho} \cdot \frac{1}{\tau^2\rho^{4l}}\cdot \log \frac{N( \rho^l)}{\rho^l}\cdot N(\rho^l) \cdot \log N(\rho^l)\cdot \log\frac{L}{\delta}.  
     \end{align*}

    \item {In particular, if $\mS$ is a unit volume hypercube in $\mathbb{R}^d$, then the total sample complexity to achieve $\varepsilon$-error value function and policy is given by }
      \begin{align*}
       O\Big( \log \frac{1}{\tau\rho}\cdot \frac{1}{\tau^2\varepsilon^{d+4}} \cdot \log \big(\frac{1}{\varepsilon}\big)^4 \cdot \log\frac{1}{\delta}\Big). 
    \end{align*} 

\end{enumerate}

\end{thm}

Theorem~\ref{thm:informal_example_converge} states that for a unit hypercube, the sample complexity of the instance of EIS scales as $\widetilde{O}\big({\varepsilon^{-(4+d)}}\big)$ (omitting the logarithmic factor). Note that the minimax lower bound in Theorem~\ref{thm:lower_bound} scales as $\widetilde{\Omega}\big({\varepsilon^{-(2+d)}}\big)$. Hence, in terms of the dependence on the dimension, the instance we consider here is nearly optimal. We note that the $\widetilde{O}\big({\varepsilon^{-(4+d)}}\big)$ sample complexity results from two parts: the MCTS contributes a sample complexity scaling as $\varepsilon^{-2}$ due to simulating the search tree, while nearest neighbor requires $\varepsilon^{-(2+d)}$ samples due to the need of sufficiently many good neighbors. 
Obtaining tighter bound with potentially more powerful improvement module or supervised learning module such as neural networks is an interesting future avenue.

%% file: ais_tex_files/conclusion.tex
\section{Conclusion}\label{sec:conclusion}

In this paper, we take theoretical steps towards understanding reinforcement learning for zero-sum turn-based Markov games. 
We develop the Explore-Improve-Supervise (EIS) method with three intuitive modules intertwined carefully. Such an abstraction of three key modules allows us to isolate the fundamental principles from the implementation details. Importantly, we identify conditions for successfully finding the optimal solutions, backed by a concrete instance satisfying those conditions. Overall, the abstraction and the generic properties developed in this paper could serve as some guidelines, with the potential of finding broader applications with different instantiations. Finally, it would be interesting  to extend this framework to general Markov games with simultaneous moves. We believe the generic modeling techniques in Section~\ref{subsec:properties} could be applied, but a key challenge is to develop an improvement module with rigorous non-asymptotic guarantees that satisfies the desired property. We believe that addressing this challenge and formally establishing the framework is a fruitful future direction.

%% file: ais_tex_files/app_pre_facts.tex
\section{Preliminary Facts}\label{sec:appendix_prelim}

The following inequalities are used for developing our technical results:

{\bf Jensen's Inequality:} Let $X$ be a random variable and $\phi$ be a convex function, then $\phi(\mathbb{E}[X])\leq \mathbb{E}[\phi(X)]$.

{\bf Pinsker's Inequality:} Let $\mu$ and $\nu$ be two probability distributions, then the total variation distance $\textup{TV}(\mu,\nu)$ and the KL divergences $D_{\text{KL}}(\mu \Vert \nu)$ satisfy the bound
\begin{equation*}
\textup{TV}(\mu,\nu)\leq \sqrt{\frac{1}{2}D_{\text{KL}}(\mu,\nu)}
\end{equation*}
Note that if $\mu$ and $\nu$ are discrete distributions, then $\textup{TV}(\mu,\nu)=\frac{1}{2}\sum_{\omega\in\Omega}|\mu(\omega)-\nu(\omega)| = \frac{1}{2} \| \mu - \nu \|_1$, where $\|\cdot \|_1$ denotes the total variation (or $\ell_1$) norm.

{\bf Reverse Pinsker's Inequality:} If $\mu$ and $\nu$ are discrete distributions defined on a common finite set $\Omega$, then
\begin{equation*}
\textup{TV}(\mu,\nu)\geq \sqrt{\frac{\alpha_\nu}{2}D_{\text{KL}}(\mu,\nu)},
\end{equation*}
where $\alpha_\nu := \min_{\omega \in \Omega}  \nu(\omega)$. 

{\bf Log-sum inequality:} Suppose that $a_{i}\ge0,b_{i}\ge0,\forall\: i$. We have
\[
\left(\sum_{i}a_{i}\right)\log\frac{\sum_{i}a_{i}}{\sum_{i}b_{i}}\le\sum_{i}a_{i}\log\frac{a_{i}}{b_{i}}.
\]

%% file: ais_tex_files/app_eis_convergence.tex
\section{Proof of Theorem~\ref{thm:eis_convergence}}
\begin{proof} With the three detailed properties, the proof is conceptually straightforward. At each iteration, the improvement module would produce better estimates for the explored states, by factors of $\zeta_v$ and $\zeta_p$. The exploration continues until one of the desired representative sets in $\mathcal{E}$ has been visited, and the exploration property guarantees that the exploration time will be finite. The current iteration then ends by calling the supervised learning module to generalize the improvement to the entire state space. In what follows, we make these statements formal. 
 
Let us first introduce some notion. We will use the term iteration to refer to a complete round of improvement, exploration and supervised learning (cf. Line 2 of Algorithm \ref{alg:eis}). 
In general, at each iteration, we use a superscript $(l)$ to denote quantities relevant to the $l$-th iteration, except that for the supervised learning module, we follow the convention in the paper and use a subscript $l$, i.e., $f_l=(V_l,\pi_l)$. We denote by $Z^{(l)}$ all the information
during the $l$-th iteration. Let $\{\mathcal{F}^{(l)}\}$
be the sigma-algebra generated by the stochastic process $\{Z^{(l)}\}$, where the randomness comes from the environment and any randomness that may be used in the three modules. Let $\omega_v^{(l)}$ and $\omega_p^{(l)}$ be the estimation errors of the model at the beginning of the iteration, i.e.,
\begin{align*}
    \omega_{v}^{(l)}&\triangleq\E\Big[\big\Vert V_{{l-1}}-V^{*}\big\Vert_{\infty}\big],\\
    \omega_p^{(l)}&\triangleq \sup_{s\in\mS}\E\Big[D_{\textup{KL}}\big({\pi}_{l-1}(\cdot|s)\Vert P_\tau^{*}(\cdot|s)\big)\Big].
\end{align*}
We use $\mathcal{D}^{(l)}=\Big\{\big(s_{i},\hat{V}^{(l)}(s_{i}),\hat{\pi}^{(l}(\cdot|s_{i})\big)\Big\}_{i=1}^{n_l}$
to denote the set of training data generated by the exploration module and querying the improvement module during the $l$-th iteration. Let $S^{(l)}=\{s_{i}\}_{i=1}^{n_{l}}$
be the set of states visited by the exploration module. Correspondingly, the estimation errors for the value function and the optimal policy after querying the improvement module are denoted by $\varepsilon_{v}^{(l)}$
and $\varepsilon_{p}^{(l)}$, respectively: 
\begin{align*}
\varepsilon_{v}^{(l)} & =\sup_{s\in S^{(l)}}\E\Big[|\hat{V}^{(l)}(s)-V^{*}(s)|\Big],\\
\varepsilon_{p}^{(l)} & =\sup_{s\in S^{(l)}}\E\Big[D_{\textup{KL}}\big(\hat{\pi}^{(l)}(\cdot|s)\Vert P_{\tau}^{*}(\cdot|s)\big)\Big].
\end{align*}

At the $l$-th iteration, the supervised learning modules takes the outputs
of the improvement module, $\mathcal{D}^{(l)}$, as the training data. Let $\xi_{v}^{(l)}$
and $\xi_{p}^{(l)}$ denote the estimation errors for the new model $f_l=(V_l,\pi_l)$, after querying the supervised learning module:
\begin{align*}
\xi_{v}^{(l)} & =\E\Big[\big\Vert V_{{l}}-V^{*}\big\Vert_{\infty}\Big],\\
\xi_{p}^{(l)} & =\sup_{s\in S}\E\Big[D_{\textup{KL}}\big(\pi_{{l}}(\cdot|s)\Vert P_{\tau}^{*}(\cdot|s)\big)\Big].
\end{align*}
Note that by definition, $\xi_v^{(l)}=\omega_v^{(l+1)}$ and $\xi_p^{(l)}=\omega_p^{(l+1)}$. 

First, the improvement property of the improvement module (cf. Property~\ref{property:improvement}) implies that
\begin{align}
    \varepsilon^{(l)}_v &\leq \zeta_v \omega_{v}^{(l)}, \label{eq:thm2_epsilon_1}\\
    \varepsilon^{(l)}_p &\leq \zeta_p \omega_{p}^{(l)}.\label{eq:thm2_epsilon_2}
\end{align}

For the supervised learning module, according to the generalization property (cf. Property~\ref{property:sl}), when the size of training set $n_l$
is sufficiently large and the sampled states $S^{(l)}=\{s_{i}\}_{i=1}^{n_l}$
are representative  of the state space, the same order of accuracy of the training data will be generalized to the entire state space. For now, let us assume that this is the case; we will come back to verify the generalization bound in Property~\ref{property:sl} can indeed be satisfied by $S^{(l)}$.
Then, the following bounds hold:
\begin{align*}
\xi_{v}^{(l)} & \leq c_{v}\varepsilon^{(l)}_v,  \\
\xi_{p}^{(l)} & \leq c_{p} \varepsilon^{(l)}_p.
\end{align*}
Hence 
\begin{align*}
    \omega_v^{(l+1)}=\xi_{v}^{(l)} \leq c_{v}\zeta_{v} \omega_v^{(l)},\\
    \omega_p^{(l+1)}=\xi_{p}^{(l)} \leq c_{p}\zeta_{p} \omega_p^{(l)}.
\end{align*}
Therefore, when querying the improvement module, if we select the improvement factors to be
\begin{equation}
\zeta_v=\rho/c_v\textrm{ and } \zeta_p=\rho/c_p,\label{eq:thm2_zeta}
\end{equation}
then we have
\begin{align}
    \omega_v^{(l+1)}=\xi_{v}^{(l)} \leq \rho \omega_v^{(l)}, \label{eq:thm2_omega_1}\\
    \omega_p^{(l+1)}=\xi_{p}^{(l)} \leq \rho \omega_p^{(l)}.\label{eq:thm2_omega_2}
\end{align}
It is worth taking note of the fact that $c_v$ and $c_p$ would be larger than $1$ (cf. Property~\ref{property:sl}): a reasonable supervised learning model may generalize the same order of accuracy as training data, but unlikely for it be smaller; hence, 
$\zeta_v$ and $\zeta_p$ are required to be smaller than $1$ in Property~\ref{property:improvement} so that $\rho < 1$.

By definition, $\omega_v^{(1)}=C_{0,v}$ and $\omega_p^{(1)}=C_{0,p}$. Therefore, we have the desired inequalities:
\begin{align*}
    &\E\Big[\big\Vert V_{{L}}-V^{*}\big\Vert_{\infty}\big]\leq C_{0,v}\rho^L\\
    & \E\Big[D_{\textup{KL}}\big({\pi}_{L}(\cdot|s)\Vert P_\tau^{*}(\cdot|s)\big)\Big] \leq C_{0,p}\rho^L,\quad \forall s\in \mathcal{S}.
\end{align*}

Finally, to complete the proof, as we mentioned before, we need to verify that we could sample enough representative states at each iteration in finite time steps. This is indeed guaranteed by the exploration property. In particular, note that at the $l$-th iteration, we would like to sample enough representative states that are of errors $\zeta_v\omega_v^{(l)}$ and $\zeta_p\omega_p^{(l)}$  for the value and policy functions (cf. Eqs.~(\ref{eq:thm2_epsilon_1}) and (\ref{eq:thm2_epsilon_2})). By a recursive argument (cf. Eqs.~(\ref{eq:thm2_omega_1}) and (\ref{eq:thm2_omega_2})), it is not hard to see that at the $l$-th iteration, we need to query the exploration module until the sampled states, $S^{(l)}=\{s_{i}\}_{i=1}^{n_{l}}$, contain one of the representative sets in $\mathcal{E}(\tau,\zeta_vC_{0,v}\rho^{l-1},\zeta_pC_{0,p}\rho^{l-1})$, i.e., we immediately stop querying the exploration module at time $n_l$ when the following holds:
\begin{equation*}
    \exists \:\hat{S}\in\mathcal{E}(\tau,\zeta_vC_{0,v}\rho^{l-1},\zeta_pC_{0,p}\rho^{l-1}) \textrm{ such that }\hat{S}\subset \{s_{i}\}_{i=1}^{n_l},
\end{equation*}
where $\zeta_v$ and $\zeta_p$ are given by Eq.~(\ref{eq:thm2_zeta}).
From the exploration property, we know that $\E[T(\tau,\zeta_vC_{0,v}\rho^{l-1},\zeta_pC_{0,p}\rho^{l-1})]$ is finite, which implies that $n_l$ is also finite with high probability. Therefore, we are guaranteed that the training data $D^{(l)}$ contains one of the representative sets, and hence the supervised learning module will generalize at each iteration. This completes the proof of Theorem \ref{thm:eis_convergence}.

\end{proof}

%% file: ais_tex_files/app_eis_sample.tex
\section{Proof of Proposition~\ref{prop:eis_sample_complexity}}
To prove Proposition \ref{prop:eis_sample_complexity}, we first establish the following useful lemma:
\begin{lem}
\label{lem:exploration_prob}
Consider the exploration module and suppose that $\E[T(\tau,\varepsilon_{1,v},\varepsilon_{1,p})]\leq B(\tau,\varepsilon_{1,v},\varepsilon_{1,p})$. Then, with probability at least $1-\delta$, 
\begin{equation}
    T(\tau,\varepsilon_{1,v},\varepsilon_{1,p})\leq e\cdot B(\tau,\varepsilon_{1,v},\varepsilon_{1,p})\cdot\log\frac{1}{\delta}.
\end{equation}
\end{lem}
\begin{proof}[Proof of Lemma \ref{lem:exploration_prob}] 
Consider a total time steps of $n=eB(\tau,\varepsilon_{1,v},\varepsilon_{1,p})\log\frac{1}{\delta}$. All the states, $\{s_i\}_{i=1}^n$, are sampled via querying the exploration module. Let us divide the total time steps $n$ into $M\triangleq\log(1/\delta)$ segments, each consisting of $h\triangleq eB(\tau,\varepsilon_{1,v},\varepsilon_{1,p})$ states. Denote by $\mS(m)$ the set of states in the $m$-th segment, i.e., $\mS(m)=\{s_i\}_{i=(m-1)h}^{mh-1}$. The key idea of the proof is to argue that with high probability, at least one of the sets $\mS(m)$, $m=1,2,\dots,M$ will contain a representative set in $\mathcal{E}(\tau,\varepsilon_{1,v},\varepsilon_{1,p})$.

Denote by $E_m$ the event that the $m$-th segment does not contain any the representative sets, i.e.,
\[E_m=\{\:\nexists\: \hat{S}\in\mathcal{E}(\tau,\varepsilon_{1,v},\varepsilon_{1,p})\textrm{ such that }\hat{S}\in\mS(m)\}.\]
Let $\mathcal{F}_m$ be the filtration containing information untill the end of segment $m$.
Since $\E[T(\tau,\varepsilon_{1,v},\varepsilon_{1,p})]\leq B(\tau,\varepsilon_{1,v},\varepsilon_{1,p})$, by Markov inequality, we have, 
\begin{equation*}
    \mathbb{P}\big(T(\tau,\varepsilon_{1,v},\varepsilon_{1,p})\geq h+1\big)\leq \frac{\E[T(\tau,\varepsilon_{1,v},\varepsilon_{1,p})]}{h+1}\leq\frac{B(\tau,\varepsilon_{1,v},\varepsilon_{1,p})}{h}=\frac{1}{e}.
\end{equation*}
This then implies that 
\begin{equation*}
    \mathbb{P}(E_m|\mathcal{F}_{m-1})\leq \frac{1}{e},\quad m\in[M].
\end{equation*}
Therefore,
\begin{align*}
    \mathbb{P}\Big( T(\tau,\varepsilon_{1,v},\varepsilon_{1,p})> e\cdot B(\tau,\varepsilon_{1,v},\varepsilon_{1,p})\cdot\log\frac{1}{\delta}\Big)&\leq\mathbb{P}(\:\forall\: m\in[M],\:E_m \textrm{ occurs})\\
    &\leq (\frac{1}{e})^m\\
    &=\delta,
\end{align*}
which completes the proof of Lemma \ref{lem:exploration_prob}.
\end{proof}

\begin{proof}[Proof of Proposition \ref{prop:eis_sample_complexity}]
With Lemma \ref{lem:exploration_prob}, we are now ready to prove Proposition \ref{prop:eis_sample_complexity}. This is achieved by simply counting the sample complexity for each of the $L$ iterations. As discussed in the convergence proof of Theorem \ref{thm:eis_convergence}, at the $l$-th iteration, we need to query the exploration module until the sampled states, $S^{(l)}=\{s_{i}\}_{i=1}^{n_{l}}$, contains one of the representative sets in $\mathcal{E}(\tau,C_{0,v}\rho^{l}/c_v,C_{0,p}\rho^{l}/c_p)$. For each of the explored states, a query of the improvement module incurs a deterministic sample complexity of $\kappa\Big(\tau,C_{0,v}\rho^{l-1},C_{0,p}\rho^{l-1},\frac{\rho}{c_v},\frac{\rho}{c_p}\Big)$, for the required improvement factors $\zeta_v=\rho/c_v$ and $\zeta_p=\rho/c_p$. Let us now apply Lemma \ref{lem:exploration_prob}. Then, we know 
\begin{equation*}
    \mathbb{P}\Big(n_l\leq e\cdot B\Big(\tau,\frac{C_{0,v}\rho^{l}}{c_v},\frac{C_{0,p}\rho^{l}}{c_p}\Big)\cdot\log\frac{L}{\delta}\Big)\geq 1-\frac{\delta}{L}.
\end{equation*}
That is, with probability at most ${\delta}/{L}$, the sample complexity of the $l$-th iteration is larger than
\[\kappa\Big(\tau,C_{0,v}\rho^{l-1},C_{0,p}\rho^{l-1},\frac{\rho}{c_v},\frac{\rho}{c_p}\Big)\cdot B\Big(\tau,\frac{C_{0,v}\rho^{l}}{c_v},\frac{C_{0,p}\rho^{l}}{c_p}\Big)\cdot e\cdot\log\frac{L}{\delta}.\]
Finally, applying union bound over the $L$ iterations, we have
\begin{align*}
    &\mathbb{P}\Big(\:\exists\: l\in[L] \textrm{ such that }n_l> e\cdot B\Big(\tau,\frac{C_{0,v}\rho^{l}}{c_v},\frac{C_{0,p}\rho^{l}}{c_p}\Big)\cdot\log\frac{L}{\delta}\Big)\\
    \leq& \sum_{l=1}^L\mathbb{P}\Big(n_l> e\cdot B\Big(\tau,\frac{C_{0,v}\rho^{l}}{c_v},\frac{C_{0,p}\rho^{l}}{c_p}\Big)\cdot\log\frac{L}{\delta}\Big)\\
    \leq&L\cdot\frac{\delta}{L}\\
    =&\delta.
\end{align*}
Therefore, with probability at least $1-\delta$, for every $l\in[L]$, $n_l\leq e\cdot B\Big(\tau,\frac{C_{0,v}\rho^{l}}{c_v},\frac{C_{0,p}\rho^{l}}{c_p}\Big)\cdot\log\frac{L}{\delta}$. Equivalently, with probability at least $1-\delta$, the total sample complexity is upper bounded by 
\begin{equation*}
    \sum_{l=1}^L\kappa\Big(\tau,C_{0,v}\rho^{l-1},C_{0,p}\rho^{l-1},\frac{\rho}{c_v},\frac{\rho}{c_p}\Big)\cdot B\Big(\tau,\frac{C_{0,v}\rho^{l}}{c_v},\frac{C_{0,p}\rho^{l}}{c_p}\Big)\cdot e\cdot\log \frac{L}{\delta}. 
\end{equation*}
\end{proof}

%% file: ais_tex_files/app_lower_bound.tex
\section{Proof of Theorem~\ref{thm:lower_bound}} \label{appendix:proof_lower_bound}

The recent work~\cite{shah2018qlnn} establishes a lower bound on the sample complexity for reinforcement learning algorithms on MDPs. We follow a similar argument to establish a lower bound on the sample complexity for two-player zero-sum Markov games. We provide the proof for completeness. The key idea is to connect the problem of estimating the value function to the problem of non-parametric regression, and then leveraging known minimax lower bound for the latter. In particular, we show that a class of non-parametric regression problem can be embedded in a Markov game problem, so any algorithm for the latter can be used to solve the former. Prior work on non-parametric regression~\citep{tsybakov2009nonparm,stone1982optimal} establishes that a certain number of observations is \emph{necessary} to achieve a given accuracy using \emph{any} algorithms, hence leading to a corresponding necessary condition for the sample size of estimating the value function in a Markov game problem. We now provide the details. 

\noindent{\bf Step 1. Non-parametric regression}

Consider the following non-parametric regression problem:
Let $ \mS:=[0,1]^{d} $ and assume that we have $T$ independent pairs of random variables $(x_{1},y_{1}),\ldots,(x_{T},y_{T})$
such that 
\begin{equation}
\E\left[y_{t}|x_{t}\right]=f(x_{t}),\qquad x_{t}\in\mS \label{eq:regression}
\end{equation}
where $x_{t}\sim\text{uniform}(\mS)$ and $f:\mS\to\real$ is the
unknown regression function. Suppose that the conditional distribution of $ y_t $ given $ x_t=x $ is a Bernoulli distribution with mean $ f(x) $. We also assume that $ f $ is $1 $-Lipschitz continuous with respect to the Euclidean norm, i.e., 
\[
 |f(x)-f(x_0)|\leq  \vert x-x_0 \vert, \quad \forall x,x_0\in \mS.
 \]
Let $ \mathcal{F} $ be the collection of all $ 1 $-Lipschitz continuous function on $ \mathcal{X}$, i.e., 
\[
\mathcal{F}=\left\{ h|\text{\ensuremath{h} is a 1-Lipschitz function on \ensuremath{\mS}}\right\},
\] 
The goal is to estimate~$f$ given the observations $(x_{1},y_{1}),\ldots,(x_{T},y_{T})$ and the prior knowledge that $ f\in \mathcal{F} $. 

It is easy to verify that the above problem is a special case of the non-parametric regression problem considered in the work by \cite{stone1982optimal} (in particular, Example~2 therein).
Let $ \hat{f}_T $ denote an arbitrary (measurable) estimator of $ f $ based on the training samples $(x_{1},y_{1}),\ldots,(x_{T},y_{T})$.
By Theorem~1 in~\cite{stone1982optimal}, we have the following result: there exists a $ c>0 $ such that 
\begin{align}
\lim_{T\to\infty}\inf_{\hat{f}_{T}}\sup_{f\in\mF}\P\bigg(\big\Vert \hat{f}_{T}-f\big\Vert _{\infty}\ge c\Big(\frac{\log T}{T}\Big)^{\frac{1}{2+d}}\bigg)=1,
\end{align}
where infimum is over all possible estimators $ \hat{f}_T $. 

Translating this result to the non-asymptotic regime, we obtain the
following theorem.
\begin{thm}
	\label{thm:regression_lower_bound}Under the above assumptions,
	for any number $\delta\in(0,1)$, there exits some numbers $ c>0 $ and $T_{\delta}$
	such that 
	\[
	\inf_{\hat{f}_{T}}\sup_{f\in\mF}\P\bigg(\big\Vert \hat{f}_{T}-f\big\Vert _{\infty}\ge c\Big(\frac{\log T}{T}\Big)^{\frac{1}{2+d}}\bigg) \ge \delta, \qquad\text{for all \ensuremath{T\ge T_{\delta}}}.
	\]
\end{thm}

\medskip
\noindent{\bf Step 2. Two-player zero-sum Markov game}

Consider a class of (degenerate) two-player zero-sum discounted Markov game $ (\mS_1, \mS_2,$ $ \mA_1, \mA_2, r, P, \gamma) $,
where
\begin{align*}
&\mS_1  =[0,1/2]^{d},\mS_2=(1/2,1]^d,\\
&\mA_1 =\mA_2=\mA \text{ is finite},\\
&P(s,a)  \text{ is supported on a single state,}\\
&r(x,a)  =r(x)\text{ for all \ensuremath{a}},\\
&\gamma  =0.
\end{align*}
In words, the transition is deterministic, and the expected reward
is independent of the action taken and the current state. 

Let $R_{t}$ be the observed reward at step $t$. We assume that the distribution of $R_{t}$ given $x_{t}$ is $\text{Bernoulli}\big(r(x_{t})\big)$, independently of $(x_{1},x_{2},\ldots,x_{t-1})$.
The expected reward function $
r(x_{t})=\E\left[R(x_{t})|x_{t}\right]
$
is assumed to be $1$-Lipschitz and bounded.
It is easy to see that for all $x\in\mS$, $a\in\mA$, 
\begin{align}
V^{*}(x) & =r(x). \label{eq:V_r_function}
\end{align}

\smallskip
\noindent{\bf Step 3. Reduction from regression to a Markov game}

Given a non-parametric regression problem as described in Step $1$,
we may reduce it to the problem of estimating the value function $V^{*}$
of the Markov game described in Step $2$. To do this, we set
\begin{align*}
r(x) & =f(x), \qquad\forall x\in\mS
\end{align*}
and
\begin{align*}
R_{t} & =y_{t},\qquad t=1,2,\ldots,T.
\end{align*}
In this case, it follows from equations~\eqref{eq:V_r_function} that the value function is given by $V^{*}=f$. Moreover, the expected reward function $r(\cdot)$ is $1$-Lipschitz, so the assumptions of the Markov game in Step $ 2 $ are
satisfied. This reduction shows that the Markov game problem is at least as hard as
the nonparametric regression problem, so a lower bound for the latter
is also a lower bound for the former. 

Applying Theorem~\ref{thm:regression_lower_bound} yields the following result:
for any number	$\delta\in(0,1)$, there exist some numbers $c>0$ and $T_{\delta}>0$, such that 
	\[
	\inf_{\hat{V}_T}\sup_{V^*\in \mF}\P\bigg[ \big\Vert \hat{V}_T-V^{*}\big\Vert_{\infty}\ge c\left(\frac{\log T}{T}\right)^{\frac{1}{2+d}} \bigg] \ge\delta, \qquad\text{for all \ensuremath{T\ge T_{\delta}}}.
	\]
	Consequently, for any reinforcement learning algorithm $\hat{V}_T$
	and any sufficiently small $\varepsilon>0$, there exists a Markov game problem
	such that in order to achieve 
	\[
	\P\Big[\big\Vert \hat{V}_T-V^{*}\big\Vert_{\infty}<\varepsilon\Big]\ge 1-\delta,
	\]
	one must have 
	\[
	T\ge C'd\left(\frac{1}{\varepsilon}\right)^{2+d}\log\left(\frac{1}{\varepsilon}\right),
	\]
	where $C'>0$ is a constant.

%% file: ais_tex_files/implementation.tex
\section{Details: A Concrete Instantiation of the Key Modules}\label{appendix:example}
In Section \ref{sec:example}, we provide a sketch of our instantiation of the key modules and their informal properties. In this appendix, we close the gap by giving a detailed treatment on each of the three modules. We discuss in details each instantiation and its formal property. Combining together, we provide a precise statement on the sample complexity of the overall EIS algorithm.

\subsection{Improvement Module: MCTS}
\label{subsec:improvement_example}
Recall that the improvement module should be capable of providing improved estimates for both the value and policy functions, at the queried state $s$. Since both the value and the policy are closely related to the $Q$ function, one simple approach to simultaneously produce improved estimates is to obtain better estimates of $Q^*$ first and then construct the improved estimates of value and policy from $\hat{Q}$. We will take this approach in this example and use MCTS to obtain the $Q$ estimates.   

MCTS is a class of popular search algorithms for sequential decision-makings, by building search trees and randomly sampling the state space. It is also one of the key components underlying the success of AlphaGo Zero. Most variants of MCTS in literature uses some forms of upper confidence bound (UCB) algorithm to select actions at each depth of the search tree. Since our focus is to demonstrate the improvement property, we employ a fixed $H$-depth MCTS, which takes the current model of the value function $V_l$ as inputs and outputs a value estimate $\hat{V}(s)$ of the root node $s$. The current model $V_l$ of the value function is used for evaluating the value of the leaf nodes at depth $H$ during the Monte Carlo simulation.  This fixed depth MCTS has been rigorously analyzed in \cite{shah2019mcts} with non-asymptotic error bound for the root node, when the state transition is deterministic. 
 
We refer readers to \cite{shah2019mcts} (precisely, Algorithm 2) for the details of the pseudo code. We remark that in principle, many other variants of MCTS that has a precise error guarantee could be used instead; we choose the fixed $H$-depth variant here to provide a concrete example.

We now lay down the overall algorithm of the improvement module in Algorithm \ref{alg:mcts} below. Recall that the state transition is deterministic and the reward $r(s,a)$ could be random (cf. Section \ref{sec:game}). Given the queried state $s$, note that the Q-value estimate $\hat{Q}(s,a)$ for each $a\in\mathcal{A}$ is given by the sum of two components: (1) empirical average of the reward $r(s,a)$; (2) the estimated value $\hat{V}(s\circ a)$ for the next state, returned from calling the fixed depth MCTS algorithm with $s\circ a$ being the root node. Further recall that we use player $\pone$ as the reference (i.e., $r(s,a)\triangleq r^1(s,a)$). The module then obtains improved value estimate $\hat{V}(s)$ by taking proper max or min of the Q-value estimates $\hat{Q}(s,a)$---depending on whether $I(s)$ is player $\pone$ (maximizer) or player $\ptwo$ (minimizer)---and improved policy estimate $\hat{\pi}(\cdot|s)$ as the Boltzmann policy based on $\hat{Q}(s,a)$. It is worth mentioning that the fixed depth MCTS algorithm was designed for discounted MDP in \cite{shah2019mcts}, but extending to game setting is straightforward as in literature \citep{kocsis2006improved,kaufmann2017monte}, i.e., by alternating between max nodes (i.e., $\pone$ plays) and min nodes (i.e., $\ptwo$) for each depth in the tree. We defer details to Appendix \ref{app:improvement_mcts}, where we prove the key theorem, Theorem \ref{thm:mcts_improvement}, of our improvement module.

\begin{algorithm}[h]
   \caption{Improvement Module}
   \label{alg:mcts}
\begin{algorithmic}[1]
   \STATE {\bfseries Input:} (1) Current model $f=(V,\pi)$; (2) root node $s$. 
   \STATE \texttt{/*}~~\texttt{Q-value estimates}\texttt{*/}
   \FOR {each $a\in\mathcal{A}$}
   \STATE call the fixed depth MCTS (Algorithm 2 of \cite{shah2019mcts}) 
   with: (1) depth $H$; (2) root node $s\circ a$; (3) current model $V$ for evaluating value of the leaf nodes at depth $H$; (4) number of MCTS simulation $m$. That is,
   \begin{equation}
       \hat{V}(s\circ a) = \textrm{Fixed Depth MCTS}(H, s\circ a, V,m)
       \vspace{-0.15in}
   \end{equation}
   \STATE simulate action $a$ for $m$ times to obtain an empirical average, $\hat{r}(s,a)$ , of the reward $r(s,a)$.
   \STATE form Q-value estimate for $Q^*(s,a)$:
   \begin{equation}
       \hat{Q}(s,a) = \hat{r}(s,a) + \gamma\hat{V}(s\circ a)
   \end{equation}
   \ENDFOR
   \STATE \texttt{/*}~~\texttt{Improved value and policy estimates}\texttt{*/}
   \STATE form value and policy estimates for $V^*(s)$ and $P^*_\tau(s)$ based on the Q-value estimates, i.e.,
   \begin{equation}
       \hat{V}(s) = \max_{a\in\mathcal{A}}\hat{Q}(s, a)\textrm{ and }\hat{\pi}(a|s)=\frac{e^{\hat{Q}(s,a)/\tau}}{\sum_{a'\in\mathcal{A}}e^{\hat{Q}(s,a')/\tau}} \textrm{ for every $a\in\mathcal{A}$}.
   \end{equation}
   if $I(s)$ is player $\ptwo$, then replace $\max$ with $\min$ in the value estimate and replace $\hat{Q}$ with $-\hat{Q}$ in the policy estimate (recall that the maximizer $\pone$ is the reference).
   \STATE {\bfseries Output:} improved estimates $\hat{V}(s)$ and  $\hat{\pi}(\cdot|s)$.
\end{algorithmic}
\end{algorithm}

The following theorem states the property of this specific improvement module (Algorithm \ref{alg:mcts}). It is not hard to see that it directly implies the desired improvement property, i.e., Property \ref{property:improvement}.

\begin{thm}
\label{thm:mcts_improvement}
Suppose that the state transitions are deterministic. Given the current model $f=(V,\pi)$, a small temperature $\tau <1$, and the improvement factors $0<\zeta_v<1$ and $0<\zeta_p<1$. Suppose that the current value model, $V$, satisfies that 
\begin{align*}
    \E\Big[||V- V^*||_\infty\Big]&\leq \varepsilon_{0,v}.
\end{align*}
Then, with appropriately chosen parameters for Algorithm~\ref{alg:mcts}, for each query state $s_0\in\mathcal{S}$, $ \big(\hat{V}(s_0), \hat{\pi}(\cdot|s_0)\big) = \textrm{Improvement Module}(f,s_0)$, we have: 
\begin{enumerate}
\item $
\mathbb{E}\Big[\big|\hat{V}(s_0)-V^*(s_0)|\big|\Big]\leq\zeta_v\cdot\varepsilon_{0,v},\: \textrm{  and  }\: \mathbb{E}\Big[D_{\textup{KL}}\big(\hat{\pi}(\cdot|s_0)||P^*_\tau(\cdot|s_0)\big)\Big]\leq \zeta_p\cdot\varepsilon_{0,v}.$

\item The above is achieved with a sample complexity of \[O\Big(\big({\tau\cdot\min\{\zeta_v,\zeta_p\}\cdot\varepsilon_{0,v}}\big)^{-2}\cdot \frac{\log({\tau\cdot\min\{\zeta_v,\zeta_p\}})}{\log \gamma}\Big).\]
\end{enumerate}

\end{thm}

\subsection{Supervised Learning Module: Nearest Neighbor Regression.}

To establish the generalization property of nearest neighbor supervised learning algorithm for estimating the optimal value function and policy, we make the following structural assumption about the Markov game. Specifically, we assume that
the optimal solutions (i.e., true regression function) are smooth in
some sense.  
\begin{assumption}
\label{assu:smooth}
(A1.) The state space $\mS$ is a compact subset of $\mathbb{R}^{d}$. The
chosen distance metric $d:\mathcal{S}\times\mathcal{S}\rightarrow\mathbb{R}_{+}$
associated with the state space $\mathcal{S}$ satisfies that $(\mathcal{S},d)$
forms a compact metric space. 
(A2.) The optimal value function $V^{*}:\mathcal{S}\rightarrow\mathbb{R}$ is bounded by $V_{\max}$ and
satisfies Lipschitz continuity with parameter $L_v$, i.e., $\forall s,s'\in\mathcal{S},$
$
|V^{*}(s)-V^{*}(s')|\leq L_v d(s,s').
$
(A3.) The optimal Boltzmann policy $P^*_{\tau}$ defined in Eq.~(\ref{eq:def_opt_Boltzmann}) is Lipschitz continuous with parameter $L_p(\tau)$, i.e., $\forall s,s'\in \mathcal{S},$ $\forall a\in A,$ 
$
|P^{*}_{\tau}(a|s)-P^{*}_{\tau}(a|s')|\leq L_p(\tau)d(s,s').
$
\end{assumption}

For each $h>0$, the compact $\mS$ has a finite $h/2$-covering number $N(h)$.  There
exists a \emph{partition} of~$\mS$, $\left\{ B_{j},j\in[N(h)]\right\} $,  such
that each $B_{j}$ has a diameter at most $h$, that is, $\sup_{x,y\in B_{j}}d(x,y)\le h$.
We assume that states in the training set, $T:=\left\{ s_{i},i\in[n]\right\} $,
are sufficiently representative in the sense that for any given $h$ and $K$,
the sample size $n$ can be chosen large enough to ensure that $\left|B_{j}\cap T\right|\ge K$ for all $j\in[N(h)]$. If $T$ satisfies this condition, we call it \emph{$(h,K)$-representative}.

Given the training data, we fit a value function $\vnn: \mS\to\mathbb{R}$ using the following Nearest Neighbor type algorithm: set
\[
\vnn(s)=\frac{1}{\left|B_{j}\cap T\right|}\sum_{s_i\in B_{j}\cap T}\hat{V}(s_i), \quad \forall s\in B_{j}.
\]
For each $a\in\mathcal{A}$, a similar algorithm can be used to fit
the action probability $\pinn (a|\cdot):\mS \to[0,1]$. 
The proposition below, proved in Appendix~\ref{appendix_proof_knn}, shows that this algorithm has the desired generalization property.

To simplify the notation, we use $\varepsilon_v$ and $\varepsilon_p$ to represent the estimation errors of the value function and the policy, respectively, for the training data. That is, $\varepsilon_v\triangleq\varepsilon_{1,v}$ and $\varepsilon_p\triangleq\varepsilon_{1,p}$ in Property 2.
\begin{prop}
\label{prop:knn} Suppose Assumption~\ref{assu:smooth} holds. If the training data is representative with respect to appropriate $h>0$ and $K>0$, the above regression algorithm satisfies Property~\ref{property:sl}. In particular, if
\begin{align}
    h=\min\Big\{\frac{\varepsilon_v}{L_v},\frac{\sqrt{\varepsilon_p/2}}{L_p}\Big\}, K=\max\Big\{\frac{V_{\max}^{2}}{2\varepsilon_{v}^{2}}\log\Big(\frac{4V_{\max}N}{\varepsilon_{v}}\Big),  \frac{1}{\varepsilon_p}\log\Big(\frac{4N}{\varepsilon_p}\Big) \Big\},\label{eq:knn_parameters}
\end{align}
we have 
\begin{align*}
\E\left[\left\Vert \vnn-V^{*}\right\Vert _{\infty}\right] & \le 4 \cdot \varepsilon_{v},\\
\E\left[D_{\textup{KL}}\big(\pinn(\cdot|s)\,\Vert\,P^{*}_\tau(\cdot|s)\big)\right] & \le c\cdot\varepsilon_{p},\quad \forall s \in \mS,
\end{align*}
where the constant $c>0$ is independent of $n$, the size of the training data. 
\end{prop}

The size of a representative data set should at least scale as $KN(h).$ Consider a simple setting where the state space is a unit volume hypercube in $\mathbb{R}^d$ with $l_{\infty}$ metric. By \citep[Lemma~5.7]{wainwright2019high}, 
the covering number $ N(h) $ of $\mS$ scales as $ \Theta\big((1/h)^d\big).$ Let $\varepsilon=\min\{\varepsilon_v,\sqrt{\varepsilon_p}\}.$ Note that $h=\Theta(\varepsilon)$
Therefore, to achieve the desired generalization property, the size of the training dataset should satisfy
\begin{align*}
    n\geq KN(h)=O\Big(\frac{K}{h^d}\Big)=O\bigg(\frac{d}{\varepsilon^{d+2}}\log\frac{1}{\varepsilon} \bigg).
\end{align*}    

\subsection{Exploration Module: Random Sampling Policy}

As stated in Proposition~\ref{prop:knn}, the sampled states for nearest neighbor regression should be $(h,K)$-representative, where $h,K$ are given by Eq.~(\ref{eq:knn_parameters}). 
{We will show that a random sampling policy---uniformly sampling the state space---}is able to visit a set of $(h,K)$-representative states within a finite expected number of steps. We need to assume that the state space $\mS$ is sufficiently regular near the boundary. In particular, we impose the following assumption which is naturally satisfied by convex compact sets in $\mathbb{R}^d$, for example.
\begin{assumption}
\label{assu:regularity}
The partition $\{B_j, j \in [N(h)] \}$ of $\mS$ satisfies
\begin{equation}
    \lambda(B_j) \ge c_0 \frac{\lambda(\mS)}{N(h)}, \quad  \forall j \in [N(h)],
\end{equation}
for some constant $c_0>0$, where $\lambda$ is the Lesbegue measure in $\mathbb{R}^d$.
\end{assumption}

\begin{prop} \label{prop:random_policy}
Suppose that the state space $\mS$ is a compact subset of $\mathbb{R}^d$ satisfying Assumption~\ref{assu:regularity}. Given temperature $\tau>0$, and estimation errors $\varepsilon_{v}>0$ and $\varepsilon_{p}>0$ for the value and policy respectively, define $h,K$ as in Eq.~(\ref{eq:knn_parameters}). Then the expected number of steps to obtain a set of $(h,K)$-representative states under the random sampling policy is upper bounded by 
$$B(\tau,\varepsilon_v,\varepsilon_p)=O\left(\frac{KN(h)}{c_0} \log N(h)\right).$$
\end{prop}

\subsection{Convergence Guarantees of the Instance}

For the instance of EIS algorithm with MCTS, random sampling policy and nearest neighbor supervised learning, we have shown that each module satisfies the desired properties (cf. Theorem~\ref{thm:mcts_improvement} and Propositions~\ref{prop:knn}-\ref{prop:random_policy}). Therefore, the convergence result stated in Theorem~\ref{thm:eis_convergence} holds for the specific instance we consider here. Moreover, the non-asymptotic analysis of these three methods provides an explicit upper bound on the sample complexity of this instance. The following corollary states the precise result.

\begin{thm} \label{thm:instance_of_eis} Suppose that Assumptions~\ref{assu:smooth} and~\ref{assu:regularity} hold. For a given $\rho\in (0,1)$, {and a small $\tau<1,$} there exist appropriately chosen parameters for the instance of Algorithm~\ref{alg:eis} with MCTS, random sampling and nearest neighbor supervised learning, such that:
\begin{enumerate}
    \item The output $f_{L}=(V_L,\pi_L)$ at the end of $L$-th iteration satisfies \begin{align*}
    \E\Big[\big\Vert V_{L} -V^* \big\Vert_\infty \Big] &\leq V_{\max}\rho^L,\\
    \E\Big[D_{\textup{KL}}\big(\pi_{L}(\cdot|s) \,\Vert\, P_\tau^{*}(\cdot|s)\big)\Big] &\leq  \frac{cV_{\max}}{4} \rho^L, \quad \forall s \in \mS,
    \end{align*}
    where $c$ is the generation constant for policy in Proposition~\ref{prop:knn}.

    \item {With probability at least $1-\delta,$ the above result is achieved with sample complexity of }
    \begin{align*}
        \sum_{l=1}^{L}c' \log \frac{1}{\tau\rho} \cdot \frac{1}{\tau^2\rho^{4l}}\cdot \log \frac{N(c_4 \rho^l)}{\rho^l}\cdot N(c_4\rho^l) \cdot \log N(c_4\rho^l)\cdot \log\frac{L}{\delta}.  ,
     \end{align*}
     where $c'>0$ and $c_4>0$ are constants independent of $\rho$ and $l.$

    \item {In particular, if $\mS$ is a unit volume hypercube in $\mathbb{R}^d$, then the total sample complexity to achieve $\varepsilon$-error value function and policy is given by }
      \begin{align*}
       O\bigg( \log \frac{1}{\tau\rho}\cdot \frac{1}{\tau^2\varepsilon^{d+4}} \cdot \log \big(\frac{1}{\varepsilon}\big)^4 \cdot \log\frac{1}{\delta}\bigg). 
    \end{align*}

\end{enumerate}

\end{thm}

Theorem~\ref{thm:instance_of_eis} states that the sample complexity of the instance of EIS algorithm scales as $\widetilde{\Theta}\big(\frac{1}{\varepsilon^{4+d}}\big)$ (omitting the logarithmic factor). Note that Theorem~\ref{thm:lower_bound} implies that for any policy to learn the optimal value function within $\varepsilon$ approximation error, the number of samples required must scale as $\widetilde{\Omega}\big(\frac{1}{\varepsilon^{2+d}}\big)$. Hence in terms of the dependence on the dimension, the instance we consider here is nearly optimal. 

%% file: ais_tex_files/app_mcts.tex
\section{Example Improvement Module and Proof of Theorem \ref{thm:mcts_improvement}}

\label{app:improvement_mcts}
In this section, we formally show the improvement property of the specific example in Section \ref{subsec:improvement_example}. To this end, we first elaborate some details regarding Algorithm \ref{alg:mcts} in Appendix \ref{app:subsec_details_mcts}. We then state two useful lemmas in Appendix \ref{app:subsec_useful_lemma} and finally, we complete the proof of Theorem \ref{thm:mcts_improvement} in Appendix \ref{app:subsec_mcts_proof}.

\subsection{Details of the Improvement Module Example}
\label{app:subsec_details_mcts}
Before proving the theorem, let us first discuss some details of the improvement module (i.e., Algorithm \ref{alg:mcts}).
It is worth mentioning some necessary modifications for applying the fixed $H$-depth MCTS algorithm \citep{shah2019mcts} in Algorithm 2. In particular, the original algorithm is introduced and analyzed for infinite-horizon discounted MDPs, but extending to a game setting is straightforward as similar to the literature \citep{kocsis2006improved,kaufmann2017monte}. We now elaborate both the algorithmic and the technical extensions of the fixed $H$-depth MCTS algorithm for Markov games. 

Algorithmically, for turn-based zero-sum games, each layer in the tree would alternate between max nodes (i.e., player $\pone$'s turn) and min nodes (i.e., player $\ptwo$'s turn). For max nodes, the algorithm proceeds as usual by selecting the action with the maximum sum of the empirical average and the upper confidence term. For min nodes, the algorithm could choose the action with the minimum value of the empirical average minus the upper confidence term. More precisely, if the current node is a min node, then Line 6 (action selection) of the fixed depth MCTS algorithm in \cite{shah2019mcts} should be modified to (using the notation in \cite{shah2019mcts} to be consistent):
  \begin{equation*}
       a^{(h+1)}= \arg\min_{a\in\mathcal{A}}\frac{q^{(h+1)}(s^{(h)},a)+\gamma \tilde{v}^{(h+1)}(s^{(h)}\circ a)}{N^{(h+1)}(s^{(h)}\circ a)}-\frac{\big(\beta^{(h+1)}\big)^{1/\xi^{(h+1)}}\cdot \big(N^{(h)}(s^{(h)})\big)^{\alpha^{(h+1)}/\xi^{(h+1)}}}{\big(N^{(h+1)}(s^{(h)}\circ a)\big)^{1-\eta^{(h+1)}}}.\label{eq:alg_mcts}
   \end{equation*}
Alternatively, the algorithm could first negate the empirical average and then choose the action that maximizes the sum of the negated empirical average and the upper confidence term. With these modifications, the fixed depth MCTS algorithm could be used to estimate values for the game setting considered in this paper. 

Technically, we note that one could readily obtain the same guarantees for the fixed depth MCTS algorithm for the game setting as the algorithm for MDPs in~\cite{shah2019mcts}, by following essentially the same proof in Appendix A of~\cite{shah2019mcts}. We only remark some technical points in the following:
\begin{enumerate}
    \item The concentration results (cf. Appendix A.4 of \cite{shah2019mcts}) still hold in the game setting. The original concentration inequalities in \cite{shah2019mcts} are two-sided. Therefore, they apply to both the max nodes and min nodes.
    \item The technical results were derived for rewards that are bounded in $[0,1]$ for convenience. It is not hard to see (cf. Remark 1 in Appendix A.3 of \cite{shah2019mcts}) that the same proof applies seamlessly for our setting, i.e., bounded rewards in $[-R_{\max},R_{\max}]$.
    \item Since the original derivation was for MDPs, Lemma 4 in Appendix A.8 of \cite{shah2019mcts} used the Bellman equation for MDPs. In the game setting, it is straightforward to replace it with the Bellman equation for the Markov games:
    \begin{equation*}
            V^*(s)= 
\begin{cases}
    \max_{a\in\mA}\E[r(s,a)+\gamma V^*(s\circ a)],& \text{if $s$ is a max node;} \\
    \min_{a\in\mA}\E[r(s,a)+\gamma V^*(s\circ a)],& \text{if $s$ is a min node.}
\end{cases}
    \end{equation*}
    A similar result of Lemma 4 then readily follows.
\end{enumerate}

\subsection{Two Useful Lemmas}
\label{app:subsec_useful_lemma}
We first state two useful lemmas. The first lemma bounds the difference between the two Boltzmann policies in terms of the difference of the underlying $Q$ values that are used to construct the policies. 
\begin{lem} \label{lem:KL}
Fix a state $s\in\mathcal{S}$. Suppose that the Q-value estimates satisfy
\begin{align*}
    \mathbb{E}\Big[\big|\hat{Q}(s,a)- Q^*(s,a)\big|\Big]\leq \varepsilon,\quad\forall\: a\in\mathcal{A}.
\end{align*}
Consider the two Boltzmann policies with temperature $\tau>0$:
\begin{align*}
\hat{P}_\tau(a)  =\frac{e^{\hat{Q}(s,a)/\tau}}{\sum_{a'}e^{\hat{Q}(s,a')/\tau}},\quad
P^{*}_\tau(a)  =\frac{e^{Q^{*}(s,a)/\tau}}{\sum_{a'}e^{Q^{*}(s,a')/\tau}}.
\end{align*}
We have that
\begin{equation*}
    \mathbb{E}\big[D_\textup{KL}(\hat{P}_\tau\Vert P^{*}_\tau)\big]\leq 2|\mathcal{A}|\varepsilon/\tau.
\end{equation*}
\end{lem}
\begin{proof}
Since $s$ is fixed, we drop it for the ease of exposition. Let $C=\sum_{a'}e^{\hat{Q}(a')/\tau}$ and $C^* = \sum_{a'}e^{{Q}^*(a')/\tau}$. Then
\begin{align*}
D_\textup{KL}(\hat{P}\Vert P^{*})
& =\sum_{a}\hat{P}(a)\log\frac{\hat{P}(a)}{P^{*}(a)}\\
 & =\frac{1}{C}\sum_{a}e^{\hat{Q}(a)/\tau}\left(\log\frac{e^{\hat{Q}(a)/\tau}}{e^{Q^{*}(a)/\tau}}+\log\frac{C^{*}}{C}\right)\\
 & =\frac{1}{C}\sum_{a}e^{\hat{Q}(a)/\tau}\cdot\frac{1}{\tau}\left(\hat{Q}(a)-Q^{*}(a)\right)+\underbrace{\frac{1}{C}\sum_{a}e^{\hat{Q}(a)/\tau}}_{=1}\log\frac{C^{*}}{C}\\
 & \le\frac{1}{\tau}\cdot\underbrace{\frac{1}{C}\sum_{a}e^{\hat{Q}(a)/\tau}}_{=1}\cdot\left\Vert \hat{Q}-Q^{*}\right\Vert _{\infty} + \frac{1}{C^{*}}\cdot C^{*}\log\frac{C^{*}}{C}.
 \end{align*}
 The second term above can be bounded using the log-sum inequality (cf.\ Appendix~\ref{sec:appendix_prelim}), which gives  $ C^{*}\log\frac{C^{*}}{C} \le \sum_{a}e^{Q^{*}(a)/\tau}\log\frac{e^{Q^{*}(a)/\tau}}{e^{\hat{Q}(a)/\tau}}$. We then continue the above chain of inequalities to obtain
 \begin{align*}
 D_\textup{KL}(\hat{P}\Vert P^{*}) 
 & \le \frac{1}{\tau}\left\Vert \hat{Q}-Q^{*}\right\Vert _{\infty}+\frac{1}{C^{*}}\cdot\sum_{a}e^{Q^{*}(a)/\tau}\log\frac{e^{Q^{*}(a)/\tau}}{e^{\hat{Q}(a)/\tau}} \\
 & =\frac{1}{\tau}\left\Vert \hat{Q}-Q^{*}\right\Vert _{\infty}+\frac{1}{C^{*}}\cdot\sum_{a}e^{Q^{*}(a)/t}\cdot\frac{1}{\tau}\left(Q^{*}(a)-\hat{Q}(a)\right)\\
 & \le\frac{1}{\tau}\left\Vert \hat{Q}-Q^{*}\right\Vert _{\infty}+\underbrace{\frac{1}{C^{*}}\cdot\sum_{a}e^{Q^{*}(a)/t}}_{=1}\cdot\frac{\left\Vert \hat{Q}-Q^{*}\right\Vert _{\infty}}{\tau}\\
 & =\frac{2}{\tau}\left\Vert \hat{Q}-Q^{*}\right\Vert _{\infty}.
\end{align*}
Taking expectation, we have the bound
\begin{align*}
  \mathbb{E}\big[D_\textup{KL}(\hat{P}\Vert P^{*})\big]&\leq \frac{2}{\tau}\mathbb{E}\big[\left\Vert \hat{Q}-Q^{*}\right\Vert _{\infty}\big]\\
  &\leq \frac{2}{\tau}\mathbb{E}\big[\left\Vert \hat{Q}-Q^{*}\right\Vert _{1}\big]\\
  &\leq \frac{2}{\tau}\sum_{a\in\mathcal{A}}\mathbb{E}\big[| \hat{Q}(s,a)-Q^{*}(s,a)|\big]\leq \frac{2|\mathcal{A}|\varepsilon}{\tau},
\end{align*}
as desired.
\end{proof}

The following lemma states a generic result regarding the maximum difference of two vectors.
\begin{lem}
\label{lem:max}Consider two $n$-dimensional vectors $X=(x_{1},x_{2},\ldots,x_{n})$
and $Y=(y_{1},y_{2},\ldots,y_{n})$. We have 
\begin{align*}
 \big|\max_{i\in[n]}\{x_{i}\}-\max_{j\in[n]}\{y_{j}\}\big| &\leq\max_{k\in[n]}|x_{k}-y_{k}|,\\
 \big|\min_{i\in[n]}\{x_{i}\}-\min_{j\in[n]}\{y_{j}\}\big|& \leq \max_{k\in[n]}|x_{k}-y_{k}|.
\end{align*}

\end{lem}
\begin{proof} 
 Assume that $i^{*}\in\arg\max_{i\in[n]}\{x_{i}\},$ and $j^{*}\in\arg\max_{j\in[n]}\{y_{j}\}$.
Then 
\[
\max_{i\in[n]}\{x_{i}\}-\max_{j\in[n]}\{y_{j}\}=x_{i^{*}}-y_{j^{*}}\leq x_{i^{*}}-y_{i^{*}},
\]
and 
\[
x_{i^{*}}-y_{j^{*}}\geq x_{j^{*}}-y_{j^{*}}.
\]
Thus,
\begin{align*}
\Big|\max_{i\in[n]}\{x_{i}\}-\max_{j\in[n]}\{y_{j}\}\Big|&=\big|x_{i^{*}}-y_{j^{*}}\big|\\
&\leq\max\big\{\big|x_{i^{*}}-y_{i^{*}}\big|,\big|x_{j^{*}}-y_{j^{*}}\big|\big\}\\
&\leq\max_{k\in[n]}|x_{k}-y_{k}|.
\end{align*}

The same argument holds for the other inequality, and this completes the proof of Lemma \ref{lem:max}.

\end{proof}

\subsection{Improvement Property: Proof of Theorem \ref{thm:mcts_improvement}}
\label{app:subsec_mcts_proof}
We are now ready to prove Theorem \ref{thm:mcts_improvement}.
As discussed in Appendix \ref{app:subsec_details_mcts}, the same guarantees of the modified fixed $H$-depth MCTS algorithm for the game setting can be established as in \cite{shah2019mcts}. In the sequel, we extend these guarantees, together with the previous two lemmas, to analyze the improvement module example. In particular, we derive error bounds for the outputs of Algorithm \ref{alg:mcts}, $\hat{V}$ and $\hat{\pi}$,  and analyze the corresponding sample complexity.

Consider deterministic state transitions. The complete proof proceeds in two steps. The first step is to analyze the outputs of querying the fixed $H$-depth MCTS algorithm. Based on those outputs, as the next step we then analyze the outcomes of Algorithm \ref{alg:mcts}, $\hat{V}$ and $\hat{\pi}$ . Finally, we characterize the corresponding sample complexity of the overall process. Throughout the first two steps, since the current model $f=(V,\pi)$ may be random, let us fix a realization; we will take expectation in the end to arrive at the desired results in Theorem \ref{thm:mcts_improvement}.

\emph{{\bf Step 1:} Error bounds for outputs of the fixed depth MCTS algorithm.} In~\cite{shah2019mcts}, the authors establish the following concentration result for
the estimated value function of the root node $s$, $\hat{V}_{m}(s)$, under the fixed $H$-depth MCTS algorithm with $m$
simulations (cf. Proof of Theorem 1 in~\cite{shah2019mcts}): there exist constants $\beta>1$
and $\xi>0$ and $\eta\in[1/2,1)$ such that for every $z\geq1$,
\[
\P\Big(\Big|\hat{V}_{m}(s)-\mu_{s}^{(0)}\Big|\geq m^{\eta-1}z\Big)\leq\frac{2\beta}{z^{\xi}},
\]
where the probability is measured with respect to the randomness in the MCTS algorithm, and $\mu_{s}^{(0)}$ satisfies the following condition 
\[
|\mu_{s}^{(0)}-V^{*}(s)|\leq\gamma^{H}\varepsilon_{0}.
\]
Here, $\varepsilon_0$ denotes the error when evaluating the leaf nodes using current model, i.e., $\varepsilon_0=\Vert V-V^*\Vert _\infty$, where $V$ is the current value model. Note that $\eta\in[1/2,1)$ is a hyper-parameter for the MCTS algorithm that could be freely chosen. Throughout the proof, we set
\[\eta\triangleq \frac{1}{2}.\]
Recall that $\gamma \in (0,1)$ is the discount factor. In addition, $\xi$ is larger than $1$ (cf. Section A.6.4 of \cite{shah2019mcts}).
Therefore, for every $t\geq m^{\eta-1}$,
\[
\P\Big(\Big|\hat{V}_{(m)}(s)-\mu_{s}^{(0)}\Big|\geq t\Big)\leq\frac{2\beta}{t^{\xi}}m^{\xi(\eta-1)}.
\]
It follows that
\begin{align*}
\E\Big[\Big|\hat{V}_{m}(s)-\mu_{s}^{(0)}\Big|\Big] & =\int_{0}^{\infty}\P\Big(\Big|\hat{V}_{m}(s)-\mu_{s}^{(0)}\Big|\geq t\Big)dt\\
 & \leq\int_{0}^{m^{\eta-1}}1\cdot dt+\int_{m^{\eta-1}}^{\infty}\frac{2\beta}{t^{\xi}}m^{\xi(\eta-1)}\cdot dt\\
 & =m^{\eta-1}+2\beta m^{\xi(\eta-1)}\frac{m^{(1-\xi)(\eta-1)}}{\xi-1} &  & \xi>1\\
 & =(1+\frac{2\beta}{\xi-1})m^{\eta-1}.
\end{align*}
Thus 
\begin{align*}
\E\big[\big|\hat{V_{m}}(s)-V^{*}(s)\big|\big]&\leq\E\Big[\big|\hat{V}_{m}(s)-\mu_{s}^{(0)}\big|\Big]+|\mu_{s}^{(0)}-V^{*}(s)|\\
&\leq\gamma^{H}\varepsilon_{0}+O(m^{\eta-1}).
\end{align*}
This leads to a variant of Theorem 1 in~\cite{shah2019mcts} for the performance of the fixed depth MCTS,
as stated following.
\begin{prop}
\label{thm:MCTS_refined} Let $f=(V,\pi)$ be the current model and consider the fixed depth MCTS algorithm (with depth $H$) employed in Algorithm \ref{alg:mcts}. 
Then, for each query state $s\in S$, the following claim holds for the output $\hat{V}_{m}$ of the fixed $H$-depth MCTS with $m$ simulations:
\[
\E\big[\big|\hat{V}_{m}(s)-V^{*}(s)\big|\big]\leq\gamma^{H}\Vert V-V^*\Vert_{\infty}+O\Big(m^{\eta-1}\Big),
\]
where $\eta\in[1/2,1)$ is a constant and the expectation is taken with respect to the randomness in the MCTS simulations. 
\end{prop}

\emph{{\bf Step 2:} Error bounds for outputs of the improvement module.} Now, we are ready to obtain a non-asymptotic analysis for the outputs of Algorithm~\ref{alg:mcts}. Consider Line 4 of Algorithm \ref{alg:mcts}, where we call the fixed depth MCTS algorithm on the state $s_0\circ a$.
For each $a\in \mA$, as $m$ simulations are performed with root node
$s_{0}\circ a$ during the simulation, Proposition \ref{thm:MCTS_refined} implies that 
\begin{equation}
\E\big[\big|\hat{V}(s_{0}\circ a)-V^{*}(s_{0}\circ a)\big|\big]\leq\gamma^{H}\Vert V-V^*\Vert_{\infty}+O\Big(m^{\eta-1}\Big).\label{eq:V_next_state}
\end{equation}
Recall that state transitions are assumed to be deterministic and note that the estimated Q-value for $(s_{0},a)$ is given by (i.e., Line 6 of Algorithm \ref{alg:mcts})
\[
\hat{Q}(s_{0},a)=\hat{r}(s_{0},a)+\gamma\hat{V}(s_{0}\circ a),
\]
where $\hat{r}(s_{0},a)=\frac{1}{m}\sum_{i=1}^{m}r_{i}(s_{0},a)$
is the empirical average of the immediate rewards for playing action
$a$ when in state $s_{0}$. Note that $\{r_{i}(s_{0},a)\}_{i}$ are
independent random variables that satisfy $|r_{i}(s_{0},a)|\leq R_{\max}$.
By Hoeffding inequality, it holds that 
\[
\P\big(\big|\hat{r}(s_{0},a)-\E[r(s_{0},a)]\big|>t\big)\leq2\exp\Big(\frac{-mt^{2}}{2R_{\max}^{2}}\Big).
\]
Thus 
\begin{align*}
\E\big[\big|\hat{r}(s_{0},a)-\E[r(s_{0},a)]\big|\big] & =\int_{0}^{\infty}\P\big(\big|\hat{r}(s_{0},a)-\E[r(s_{0},a)]\big|>t\big)dt\\
 & \leq\int_{0}^{\infty}2\exp\Big(\frac{-mt^{2}}{2R_{\max}^{2}}\Big)dt\\
 & =R_{\max}\sqrt{\frac{2\pi}{m}}.
\end{align*}
It follows that 
\begin{align}
&\E\Big[\big|\hat{Q}(s_{0},a)-Q^{*}(s_{0},a)\big|\Big] \nonumber\\
= & \E\Big[\big|\hat{r}(s_{0},a)+\gamma\hat{V}(s_{0}\circ a)-\E[r(s_{0},a)]-\gamma V^{*}(s_{0}\circ a)\big|\Big]\nonumber\\
 \leq & \E\Big[\big|\hat{r}(s_{0},a)-\E[r(s_{0},a)]\big|\Big]+\gamma\E\Big[\big|\hat{V}(s_{0}\circ a)-V^{*}(s_{0}\circ a)\big|\Big]\nonumber\\
 \leq &  R_{\max}\sqrt{{2\pi}{}}m^{-1/2}+\gamma^{H}\Vert V-V^*\Vert_{\infty}+O\Big(m^{\eta-1}\Big) & & \gamma<1\nonumber\\
  \leq & \gamma^{H}\Vert V-V^*\Vert_{\infty}+O\Big(m^{\eta-1}\Big), \label{eq:app_mcts_Q_bound}
\end{align}
where the last inequality follows from the fact that $\eta\in[1/2,1).$

We now consider the value estimate $\hat{V}(s_0)$ and the policy estimate $\hat{\pi}(\cdot|s_0)$ (Line 9 of Algorithm \ref{alg:mcts}) separately:

\emph{(a) Value estimate $\hat{V}(s_0)$:} In order to obtain an error bound for the value estimate of the query state
$s_{0}$, i.e., $\hat{V}(s_0)=\max_{a\in\mA}\hat{Q}(s,a)$, we apply  Lemma \ref{lem:max} from the previous section.
For the query state $s_{0}$, if $I(s_0)$ is player $\pone$, applying Lemma \ref{lem:max} yields
\begin{align*}
\big|\hat{V}(s_{0})-V^{*}(s_{0})\big| & =\big|\max_{a\in \mA}\{\hat{Q}(s_{0},a)\}-\max_{a\in \mA}\{Q^{*}(s_{0},a)\}\big| 
  \leq\max_{a\in \mA}\big|\hat{Q}(s_{0},a)-Q^{*}(s_{0},a)\big|. 
\end{align*}
Therefore, 
\begin{align}
\E\big[\big|\hat{V}(s_{0})-V^{*}(s_{0})\big|\big] & \leq\E\big[\max_{a\in \mA}\big|\hat{Q}(s_{0},a)-Q^{*}(s_{0},a)\big|\big]\nonumber\\
 & \leq\sum_{a\in \mA}\E\big[\big|\hat{Q}(s_{0},a)-Q^{*}(s_{0},a)\big|\big]\nonumber\\
 & \leq |\mA|\Big[\gamma^{H}\Vert V-V^*\Vert_{\infty}+O\Big(m^{\eta-1}\Big)\Big].\label{eq:app_mcts_value_bound}
\end{align}
Similarly, if $I(s_0)$ is player $\ptwo$, applying Lemma \ref{lem:max} also yields the same desired result, Eq.~(\ref{eq:app_mcts_value_bound}).

\emph{(b) Policy estimate $\hat{\pi}(\cdot|s_0)$:} In order to obtain an error bound for the policy estimate of the query state
$s_{0}$, i.e., $\hat{\pi}(a|s_0)=e^{\hat{Q}(s_0,a)/\tau}/\sum_{a'\in\mA}e^{\hat{Q}(s_0,a')/\tau}$, we apply Lemma \ref{lem:KL} from the previous section. Together with Eq.~(\ref{eq:app_mcts_Q_bound}), Lemma \ref{lem:KL} yields the following bound:
\begin{align}
\E\Big[D_\textup{KL}\big(\hat{\pi}(\cdot|s)\Vert P_{\tau}^{*}(\cdot|s)\big)\Big] 
 & \leq\frac{2|\mA|}{\tau}\Big[\gamma^{H}\Vert V-V^*\Vert_{\infty}+O\Big(m^{\eta-1}\Big)\Big].\label{eq:app_mcts_policy}
\end{align}

\emph{{\bf Step 3:} Completing the proof of Theorem \ref{thm:mcts_improvement}.} Recall that by the assumption of Theorem \ref{thm:mcts_improvement}, 
\[\E\Big[||V- V^*||_\infty\Big]\leq \varepsilon_{0,v}.\]
Now, taking expectation of Eqs.~(\ref{eq:app_mcts_value_bound}) and (\ref{eq:app_mcts_policy}) over the randomness in the current model $f$, we have
\begin{align}
    \E\big[\big|\hat{V}(s_{0})-V^{*}(s_{0})\big|\big] & \leq |\mA|\Big[\gamma^{H}\varepsilon_{0,v}+O\Big(m^{\eta-1}\Big)\Big],\label{eq:app_mcts_value_exp}\\
    \E\Big[D_\textup{KL}\big(\hat{\pi}(\cdot|s)\Vert P_{\tau}^{*}(\cdot|s)\big)\Big]
 & \leq\frac{2|\mA|}{\tau}\Big[\gamma^{H}\varepsilon_{0,v}+O\Big(m^{\eta-1}\Big)\Big].\label{eq:app_mcts_policy_exp}
\end{align}
Recall that in Theorem \ref{thm:mcts_improvement}, our goal for improvement is as follows:
\begin{equation}
    \mathbb{E}\Big[\big|\hat{V}(s_0)-V^*(s)|\big|\Big]\leq\zeta_v\cdot\varepsilon_{0,v},\: \textrm{  and  }\: \mathbb{E}\Big[D_{\textup{KL}}\big(\hat{\pi}(\cdot|s_0)||P^*_\tau(\cdot|s)\big)\Big]\leq \zeta_p\cdot\varepsilon_{0,v}.\label{eq:app_mcts_goal}
\end{equation}
It is not hard to see that we can choose the parameters of the fixed depth MCTS algorithm, in particular, the depth $H$ and the number of simulations $m$, in an appropriate way such that Eqs.~(\ref{eq:app_mcts_value_exp}) and (\ref{eq:app_mcts_policy_exp}) satisfy the desired improvement bound. In particular, we could choose the depth $H$ such that
\begin{equation}
    \max\Big\{|\mA|\gamma^{H}, \frac{2|\mA|}{\tau}\gamma^{H}\Big\}\leq \frac{\min\{\zeta_v,\zeta_p\}}{2},\label{eq:app_mcts_depth}
\end{equation}
and choose the number of simulations, $m$, to be large enough such that the term $O(m^{\eta-1})$ is less than $\gamma^H\varepsilon_{0,v}$. For small temperature $\tau<1$, choosing the depth
\begin{equation*}
    H=\frac{\log\frac{\tau\min\{\zeta_v,\zeta_p\}}{4|\mA|}}{\log \gamma}
\end{equation*}
would satisfy the condition Eq.~(\ref{eq:app_mcts_depth}). Recall that the tunable hyper-parameter $\eta$ is set to $\eta=1/2$. With the above $H$, this implies that the number of simulations should be
\begin{equation*}
    m=O\Big(\Big(\frac{\tau\min\{\zeta_v,\zeta_p\}\varepsilon_{0,v}}{4|\mA|}\Big)^{\frac{1}{1-\eta}}\Big)=O\Big(\big({\tau\cdot\min\{\zeta_v,\zeta_p\}\cdot\varepsilon_{0,v}}\big)^{-2}\Big).
\end{equation*}
To summarize, we show that the desired improvement, Eq.~(\ref{eq:app_mcts_goal}), can indeed be satisfied with appropriate algorithmic parameters. Finally, regarding the sample complexity, we note that with a $H$-depth tree, each simulation of the fixed depth MCTS algorithm incurs $H$ state transitions. Therefore, the total sample complexity of querying the improvement module is $m\cdot H$, which is equal to
\begin{equation*}
O\Big(\big({\tau\cdot\min\{\zeta_v,\zeta_p\}\cdot\varepsilon_{0,v}}\big)^{-2}\cdot \frac{\log({\tau\cdot\min\{\zeta_v,\zeta_p\}})}{\log \gamma}\Big).
\end{equation*}

%% file: ais_tex_files/app_knn.tex
\section{Proof of Proposition~\ref{prop:knn}}
\label{appendix_proof_knn}

\begin{proof}
Let $h$, $K$ and $\Delta$ be positive numbers to be chosen later,
and recall that $N\equiv N(h)$ is the $h/2$-covering number of $S$.
Let $T_{j}:=B_{j}\cap T$ and $K_{j}:=\left|T_{j}\right|$. For each
$j\in[N]$, we have 
\begin{align*}
 \left|\frac{1}{K_{j}}\sum_{x\in T_{j}}\left(\hat{V}(x)-V^{*}(x)\right)\right|
 & \le\left|\frac{1}{K_{j}}\sum_{x\in T_{j}}\left(\hat{V}(x)-\E\hat{V}(x)\right)\right|+\left|\frac{1}{K_{j}}\sum_{x\in T_{j}}\left(\E\hat{V}(x)-V^{*}(x)\right)\right|\\
 & \le\left|\frac{1}{K_{j}}\sum_{x\in T_{j}}\left(\hat{V}(x)-\E\hat{V}(x)\right)\right|+\frac{1}{K_{j}}\sum_{x\in T_{j}}\E\left|\hat{V}(x)-V^{*}(x)\right| &  &\\
 & \le\left|\frac{1}{K_{j}}\sum_{x\in T_{j}}\left(\hat{V}(x)-\E\hat{V}(x)\right)\right|+\varepsilon_{v}, &  & 
\end{align*}
where the second step follows from the Jensen's inequality, and the last step follows from the premise on the training error of the value function.
To bound the first term of RHS above, we note that the $K_{j}$ random variables $\left\{ \hat{V}(x),x\in T_{j}\right\} $
are independent and bounded by $V_{\max}$.
So Hoffedings inequality ensures that 
\begin{align*}
\P\left(\left|\frac{1}{K_{j}}\sum_{x\in T_{j}}\left(\hat{V}(x)-\E\hat{V}(x)\right)\right|>\Delta\right) & \le2\exp\left(-\frac{2K_{j}\Delta^{2}}{V_{\max}^{2}}\right) \le2\exp\left(-\frac{2K\Delta^{2}}{V_{\max}^{2}}\right).
\end{align*}
Combining the last two equations and applying a union bound
over $j\in[N]$, we obtain 
\[
\P\left(\max_{j\in[N]}\left|\frac{1}{K_{j}}\sum_{x\in T_{j}}\left(\hat{V}(x)-V^{*}(x)\right)\right|>\Delta+\varepsilon_{v}\right)\le2N\exp\left(-\frac{2K\Delta^{2}}{V_{\max}^{2}}\right).
\]
Since the random variable $Z:=\max_{j\in[N]}\left|\frac{1}{K_{j}}\sum_{x\in T_{j}}\left(\hat{V}(x)-V^{*}(x)\right)\right|$
satisfies $\left|Z\right|\le2V_{\max}$, we may convert the above
inequality into an expectation bound:
\begin{align*}
\E Z & =\E\left[Z\mathbb{I}_{\{Z\le\Delta+\varepsilon_{v}\}}\right]+\E\left[Z\mathbb{I}_{\{Z>\Delta+\varepsilon_{v}\}}\right]\\
 & \le\Delta+\varepsilon_{v}+2V_{\max}\P\left(Z>\Delta+\varepsilon_{v}\right)\\
 & \le\Delta+\varepsilon_{v}+4V_{\max}N\exp\left(-\frac{2K\Delta^{2}}{V_{\max}^{2}}\right).
\end{align*}

We are now ready to bound the quantity of interest:
\begin{align*}
 & \E\sup_{s\in S}\left|\vnn(s)-V^{*}(s)\right|\\
 & =\E\max_{j\in[N]}\sup_{s\in B_{j}}\left|\frac{1}{K_{j}}\sum_{x\in T_{j}}\left(\hat{V}(x)-V^{*}(s)\right)\right|\\
 & \le\E\left[\max_{j\in[N]}\sup_{s\in B_{j}}\left|\frac{1}{K_{j}}\sum_{x\in T_{j}}\left(\hat{V}(x)-V^{*}(x)\right)\right| \right. + \left.\max_{j\in[N]}\sup_{s\in B_{j}}\left|\frac{1}{K_{j}}\sum_{x\in T_{j}}\left({V}^*(x)-V^{*}(s)\right)\right|\right]\\
 & \overset{(i)}{\le} \E\left[\max_{j\in[N]}\left|\frac{1}{K_{j}}\sum_{x\in T_{j}}\left(\hat{V}(x)-V^{*}(x)\right)\right|\right]+\max_{j\in[N]}\sup_{s\in B_{j}}L_v \frac{1}{K_{j}}\sum_{x\in T_{j}}d(x,s) &  & \\
 & \overset{(ii)}{\le} \Delta+\varepsilon_{v}+4V_{\max}N\exp\left(-\frac{2K\Delta^{2}}{V_{\max}^{2}}\right)+L_v h, &  & 
\end{align*}
where step $(i)$ holds because $V^*$ is $L_v$-Lipschitz, and step $(ii)$ holds because the sets $T_{j}\subseteq B_{j}$ have diameter at most $ h$.
Now taking $\Delta=\varepsilon_{v}$, we have
\begin{align}
  \E\sup_{s\in S}\left|\vnn(s)-V^{*}(s)\right| & \le 2\varepsilon_v+ 4V_{\max}N\exp\left(-\frac{2K\varepsilon_v^{2}}{V_{\max}^{2}}\right)+L_v h. \label{eq:vnn_error_bound}
\end{align}

We now turn to the second inequality. Under the premise on the
training error of policy, we have for each $i\in[n]$, 
\begin{align*}
\varepsilon_{p} & \ge\E\left[D_{\textup{KL}}\big(\hat{\pi}(\cdot|s_{i})\,\Vert\,P^{*}_\tau(\cdot|s_{i})\big)\right]\\
 & \ge2\E\left[\textup{TV}\left(\hat{\pi}(\cdot|s_{i}),P^{*}_\tau(\cdot|s_{i})\right)^{2}\right] &  & \text{Pinsker's inequality}\\
 & \ge2\E\left[\max_{a}\left|\hat{\pi}(a|s_{i})-P^{*}_\tau(a|s_{i})\right|^{2}\right].
\end{align*}
For each $a\in\mathcal{A}$, let us fit the 
action probability ${\pinn}(a|\cdot):\mS\rightarrow [0,1]$  using a similar Nearest Neighbor type algorithm as $\vnn$: 
\[
\pinn(a|s)=\frac{1}{\left|B_{j}\cap T\right|}\sum_{x\in B_{j}\cap T}\hat{\pi}(a|x), \quad \forall s\in B_{j}.
\]
Note that $\forall x \in T,$ $\forall a\in\mA,$ $\hat{\pi}(a|x)\in [0,1]$ and $P^{*}_\tau(a|s)\in[0,1]$. Applying a similar argument as above for the fitted action probability function $\pinn(a|\cdot)$ w.r.t.\ the squared error, we have
\begin{align*}
\E\left[\sup_{s\in S}\left|\pinn(a|s)-P^{*}_\tau(a|s)\right|^{2}\right]
    &\leq 2\Delta^2+\varepsilon_p+4N\exp(-2K\Delta^2)+2L^2_ph^2. 
\end{align*}
Choosing $\Delta=\sqrt{\varepsilon_p/2},$ we obtain that
\begin{align}
\E\left[\sup_{s\in S}\left|\pinn(a|s)-P^{*}_\tau(a|s)\right|^{2}\right]
    &\leq 2\varepsilon_p+4N\exp(-K\varepsilon_p)+2L^2_ph^2. \label{eq:pinn_error_bound}
\end{align}

Taking 
\begin{align}
    h=\min\Big\{\frac{\varepsilon_v}{L_v},\frac{\sqrt{\varepsilon_p/2}}{L_p}\Big\}, K=\max\Big\{\frac{V_{\max}^{2}}{2\varepsilon_{v}^{2}}\log\Big(\frac{4V_{\max}N}{\varepsilon_{v}}\Big),  \frac{1}{\varepsilon_p}\log\Big(\frac{4N}{\varepsilon_p}\Big) \Big\},
\end{align}
from Eqs (\ref{eq:vnn_error_bound})-(\ref{eq:pinn_error_bound}), we have the following bounds
\begin{align*}
   \E\sup_{s\in S} & \left|\vnn(s)-V^{*}(s)\right|\le 4\varepsilon_{v}, \\
   \E\sup_{s\in S} & \left|\pinn(a|s)-P^{*}_\tau(a|s)\right|^{2} \le 4\varepsilon_p.
\end{align*}
This proves the first inequality of the proposition. 

We now focus on the inequality for the policy function. By Jensen's inequality, we have
\begin{align*}
\sup_{s\in S}\E\left[\left|\pinn(a|s)-P^{*}_\tau(a|s)\right|^{2}\right] \leq \E\left[\sup_{s\in S}\left|\pinn(a|s)-P^{*}_\tau(a|s)\right|^{2}\right] \leq 4\varepsilon_p.
\end{align*}
This is equivalent to saying that for each $s\in S$,
\begin{align}
4\varepsilon_{p} & \ge\max_{a}\E\left[\left|\pinn(a|s)-P^{*}_\tau(a|s)\right|^{2}\right]\nonumber \\
 & \ge\frac{1}{\left|\mathcal{A}\right|}\E\left[\max_{a}\left|\pinn(a|s)-P^{*}_\tau(a|s)\right|^{2}\right].\label{eq:inf_bound}
\end{align}
On the other hand, for each $s\in \mS$, each $a\in \mA,$ we have the bound
\begin{align*}
P^{*}_\tau(a|s) & =\frac{\exp\left(Q^{*}(s,a)/\tau\right)}{\sum_{a'}\exp\left(Q^{*}(s,a)/\tau\right)}\\
 & \ge\frac{\exp\left(-V_{\max}/\tau\right)}{\left|\mathcal{A}\right|\exp\left(V_{\max}/\tau\right)}=:\alpha,
\end{align*}
where $\alpha>0$. Therefore, by the reverse Pinsker's inequality (cf.\ Appendix~\ref{sec:appendix_prelim})
we have 
\begin{align*}
D_{\textup{KL}}\big(\pinn(\cdot|s)\,\Vert\,P^{*}_\tau(\cdot|s)\big) & \le\frac{2}{\alpha}\textup{TV}\left(\pinn(\cdot|s),P^{*}_\tau(\cdot|s)\right)^{2}\\
 & \le\frac{\left|\mathcal{A}\right|^{2}}{\alpha}\max_{a}\left|\pinn(a|s)-P^{*}_\tau(a|s)\right|^{2},
\end{align*}
Combining with the bound (\ref{eq:inf_bound}), we obtain that 
\[
\E\left[D_{\textup{KL}}\big(\pinn(\cdot|s)\,\Vert\,P^{*}_\tau(\cdot|s)\big)\right]\le\underbrace{\frac{\left|\mathcal{A}\right|^{2}}{\alpha}\cdot\left|\mathcal{A}\right|\cdot 4}_{=:c}\cdot\varepsilon_{p}.
\]
This completes the proof of the second inequality.
\end{proof}

%% file: ais_tex_files/app_random_policy.tex
\section{Proof of Proposition~\ref{prop:random_policy}} \label{appendix_random_policy}

In the sequel, we use the shorthand $N\equiv N(h)$ and refer to each $B_j$ as a ball. 
For each integer $t\ge1$,
let $W_{t}$ be the number of balls visited up to time
$t$. Let $T\triangleq\inf\left\{ t\ge1:W_{t}=N\right\} $ be the
first time when all balls are visited. For each $w\in\{1,2,\ldots,N\}$,
let $T_{w}\triangleq \inf\left\{ t\ge1:W_{t}=w\right\} $ be the the first
time when the $w$-th ball is visited, and let $D_{w}\triangleq T_{w}-T_{w-1}$
be the time to visit the $w$ ball after $(w-1)$ balls have been
visited.  We use the convention that $T_{0}=D_{0}=0$. By definition,
we have $T=\sum_{w=1}^{N}D_{w}.$ 

When $w-1$ balls have been visited, the probability of visiting a \emph{new} ball is at least 
\begin{align*}
	&\min_{I\subseteq[N],\left|I\right|=N-w+1}\P\Big\{ s_{T_{w-1}+1}\in\bigcup_{i\in I}B_{i}|s_{T_{w-1}}\Big\}  \\
	=& \min_{I\subseteq[N],\left|I\right|=N-w+1}\sum_{i\in I}\P\left\{ s_{T_{w-1}+1} \in B_{i}|s_{T_{w-1}}\right\} \\
	\ge& (N-w+1)\min_{i\in[N]} \P \left\{ s_{T_{w-1}+1} \in B_{i}|s_{T_{w-1}}\right\} \\
	\ge& (N-w+1)\cdot \frac{c_0}{N},
\end{align*}
where the last inequality follows from the regularity assumption.
Therefore, the time to visit the $w$-th pair after $w-1$ pairs have been visited, $D_{w},$ is stochastically dominated by a geometric random
variable with mean at most $\frac{N}{(N-w+1)c_0}.$
It follows that 
\begin{align*}
	\E T & =\sum_{w=1}^{N}\E D_{w}\le\sum_{w=1}^{N}\frac{N}{(N-w+1)c_0} =O\left(\frac{N}{c_0}\log N\right).
\end{align*}
This prove that the expected time to sample a $(h,1)$-representative set is upper bounded by $O\left(\frac{N}{c_0}\log N\right)$.
Note that if the trajectory samples $K$ $(h,1)$-representative sets, then each ball must be visited at least $K$ times. Therefore, $K\cdot \E T$ gives an upper bound for the expected time to sample a $(h,K)$-representative set, hence
\begin{align*}
    B(\tau, \varepsilon_v, \varepsilon_p) = O\left(\frac{KN}{c_0}\log N\right).
\end{align*}

%% file: ais_tex_files/app_instance_sample.tex
\section{Proof of Theorem~\ref{thm:instance_of_eis}}

We will reuse the notation introduced in the proof of Theorem~\ref{thm:eis_convergence}. We initialize the value model  $V_0(s)=0,$ $\forall s\in \mS$. Hence $\big\Vert V_{{0}}-V^{*}\big\Vert_{\infty}\leq V_{\max}.$ Consider the $l$-th iteration.
Let $\omega_{v}^{(l)}$ and $\omega_{p}^{(l)}$ denote the estimation errors for the model $f_{l-1}=(V_{l-1},\pi_{l-1})$, at the beginning of $l$-th iteration:
\begin{align*}
\omega_{v}^{(l)} & =\E\Big[\big\Vert V_{{l}}-V^{*}\big\Vert_{\infty}\Big],\\
\omega_{p}^{(l)} & =\sup_{s\in S}\E\Big[D_{\text{KL}}\big(\pi_{{l}}(\cdot|s)\Vert P_{\tau}^{*}(\cdot|s)\big)\Big].
\end{align*}

Let $S^{(l)}=\{s_{i}\}_{i=1}^{n_{l}}$
be the set of states visited by the random sampling policy during the $l$-th iteration. We require $S^{(l)}$ to be $(h^{(l)},K^{(l)})$-representative, where 
\begin{align}
    \xi^{(l)} &\triangleq \frac{V_{\max}\rho^l}{4} \label{eq:knn_parameter_1}\\
    h^{(l)}&=\min\Big\{\frac{\xi^{(l)}}{L_v},\frac{\sqrt{\xi^{(l)}/2}}{L_p}\Big\}\label{eq:knn_parameter_h} \\ 
    K^{(l)} &=\max\Big\{\frac{V_{\max}^{2}}{2\big(\xi^{(l)}\big)^{2}}\log\Big(\frac{4V_{\max}N(h^{(l)})}{\xi^{(l)}}\Big),  \frac{1}{\xi^{(l)}}\log\Big(\frac{4N(h^{(l)})}{\xi^{(l)}}\Big) \Big\}. \label{eq:knn_parameter_K}
\end{align}

We use $\mathcal{D}^{(l)}=\Big\{\big(s_{i},\hat{V}^{(l)}(s_{i}),\hat{\pi}^{(l}(\cdot|s_{i})\big)\Big\}_{i=1}^{n_l}$
to denote the set of training data generated by querying MCTS. Consider choosing the depth $H^{(l)}$ and simulation number $m^{(l)}$ parameters for MCTS at the $l$-th iteration as follows:
\begin{align}
    H^{(l)} & =\frac{\log \frac{\tau \rho}{16|\mA|}}{\log \gamma}, \label{eq:mcts_parameter_H}\\
    m^{(l)} & = c_1 \Big(\frac{\tau V_{\max}\rho^l}{16|\mA|}\Big)^{-2}, \label{eq:mcts_parameter_m}
\end{align}
where $c_1$ is a sufficiently large constant. By Theorem~\ref{thm:mcts_improvement}, for each query state $s_i$, the output $(\hat{V}(s_i),\hat{\pi}(\cdot|s_i))$ from MCTS satisfies
\begin{align*}
\mathbb{E}\Big[\big|\hat{V}(s_i)-V^*(s_i)|\big|\Big]\leq\frac{\rho}{4}\omega_v^{(l)},\: \textrm{  and  }\: \mathbb{E}\Big[D_{\mbox{KL}}\big(\hat{\pi}(\cdot|s_i)||P^*_\tau(\cdot|s)\big)\Big] \leq\frac{\rho}{4}\omega_v^{(l)}.
\end{align*}
That is, the improvement factors for the value function and policy, $\zeta_v$ and $\zeta_p$, of MCTS are follows:
\begin{align*}
    \zeta_v = \zeta_p = \frac{\rho}{4}.
\end{align*}

Note that the training set $\mathcal{D}^{(l)}$ have estimation error $\frac{\rho}{4}\omega_v^{(l)}$ for both value and policy, and  the sampled states $S^{(l)}$ of the training set are $(h^{(l)},K^{(l)})$-representative. By Proposition~\ref{prop:knn}, the output of nearest neighbor supervised learning at the end of $l$-th iteration satisfies the following generalization property:
\begin{align*}
\E\left[\left\Vert \vnn-V^{*}\right\Vert _{\infty}\right] & \le \rho \omega_v^{(l)},\\
\E\left[D_{\textup{KL}}\big(\pinn(\cdot|s)\,\Vert\,P^{*}(\cdot|s)\big)\right] & \le \frac{c \rho \omega_v^{(l)}}{4},\quad \forall s \in \mS.
\end{align*}

Therefore,
\begin{align*}
    \omega_v^{(l+1)} & \leq \rho \omega_v^{(l)}, \quad \mbox{and} \quad
    \omega_p^{(l+1)} \leq \frac{c}{4} \rho \omega_v^{(l)} 
\end{align*}
Since $\omega_v^{(1)}=V_{\max},$ we have 
\begin{align*}
    \omega_v^{(l+1)} & \leq V_{\max}\rho^l, \quad \mbox{and} \quad
    \omega_p^{(l+1)} \leq \frac{cV_{\max}}{4} \rho^l. 
\end{align*}
That is,
\begin{align*}
    &\E\Big[\big\Vert V_{{L}}-V^{*}\big\Vert_{\infty}\big]\leq V_{\max}\rho^l,\\
    & \E\Big[D_{KL}\big({\pi}_{L}(\cdot|s)\Vert P_\tau^{*}(\cdot|s)\big)\Big] \leq\frac{cV_{\max}}{4}\rho^l,\quad \forall s\in \mathcal{S}.
\end{align*}

During $l$-th iteration, the total sample complexity $M^{(l)}$ is given by $M^{(l)} =H^{(l)}m^{(l)}n_l.$ 
From Eqs.~(\ref{eq:knn_parameter_1})-(\ref{eq:mcts_parameter_m}), we have
\begin{align*}
    H^{(l)}&=c_2\log \frac{1}{\tau\rho},\\
    m^{(l)}&=\frac{c_3}{\tau^2\rho^{2l}},\\
    h^{(l)}&=c_4 \rho^l,\\
    K^{(l)}&= \frac{c_5}{\rho^{2l}}\log \frac{N(c_4 \rho^l)}{\rho^l},
\end{align*}
where $c_2,c_3,c_4,c_5$ are positive constants independent of $\rho$ and $l$. 
By Proposition~\ref{prop:random_policy}, we have 
\begin{align*}
    \E[n_l]&\leq O\Big(\frac{K^{(l)}N(h^{(l)})}{c_0}\log N(h^{(l)}) \Big)=c_6 K^{(l)}N(h^{(l)})\log N(h^{(l)}),
\end{align*}
where $c_6>0$ is a constant.

Following the argument in the proof of Proposition~\ref{prop:eis_sample_complexity}, we have: with probability at least $1-\delta,$
\begin{align*}
    \sum_{l=1}^{L}M^{(l)}&= \sum_{l=1}^{L}H^{(l)}m^{(l)}n_l\\
    &\leq \sum_{l=1}^{L}H^{(l)}m^{(l)}\E[n_l]e\log\frac{L}{\delta}\\
    &= \sum_{l=1}^{L} c_2 \log\frac{1}{\tau\rho} \cdot \frac{c_3}{\tau^2\rho^{2l}} \cdot c_6\frac{c_5}{\rho^{2l}}\log \frac{N(c_4 \rho^l)}{\rho^l}\cdot N(c_4\rho^l)\log N(c_4\rho^l)\cdot e\log\frac{L}{\delta} \\
    &=\sum_{l=1}^{L}c' \log \frac{1}{\tau\rho} \cdot \frac{1}{\tau^2\rho^{4l}}\cdot \log \frac{N(c_4 \rho^l)}{\rho^l}\cdot N(c_4\rho^l) \cdot \log N(c_4\rho^l)\cdot \log\frac{L}{\delta}.  
 \end{align*}
 Thus 
 \begin{align*}
  \sum_{l=1}^{L}M^{(l)}&= O\Big(L\cdot \log \frac{1}{\tau\rho}\cdot \frac{1}{\tau^2\rho^{4L}}\cdot \log \frac{N(\rho^L)}{\rho^L}\cdot N(\rho^L) \cdot \log N(\rho^L)\cdot \log\frac{L}{\delta}\Big).
\end{align*}

Given $0<\rho<1$, for $L=\Theta(\log \frac{1}{\varepsilon})$, i.e., $\rho^L\asymp\varepsilon$, we have 
\begin{align*}
  \sum_{l=1}^{L}M^{(l)} & = O\Big(\log \frac{1}{\varepsilon} \cdot \log \frac{1}{\tau\rho} \cdot \frac{1}{\tau^2\varepsilon^4}\cdot \log \frac{N(\varepsilon)}{\varepsilon}\cdot N(\varepsilon) \cdot \log N(\varepsilon)\cdot \log\frac{\log \frac{1}{\varepsilon}}{\delta}\Big).
\end{align*} 

In particular, if the state space $\mS$ is unit hypercube in $\mathbb{R}^d$, we have $N(\varepsilon)=O\big(\varepsilon^{-d}\big).$ Therefore,
\begin{align*}
  \sum_{l=1}^{L}M^{(l)} & = O\bigg( \log \frac{1}{\tau\rho}\cdot \frac{1}{\tau^2\varepsilon^{d+4}} \cdot \log \big(\frac{1}{\varepsilon}\big)^4 \cdot \log\frac{1}{\delta}\bigg). 
\end{align*} 

%% file: app_experiment.tex
\section{Empirical Results: A Case-Study}
\label{app:empirical}
To understand how well the EIS method does for turn-based Markov games, we consider a simple game as a proof of concept. We shall design a simple non-deterministic game and apply our EIS framework. As mentioned in Section \ref{sec:game}, the sparse sampling oracle~\citep{kearns2002sparse} could be used for the improvement module in this case. This oracle is simple but suffices to convey the insights. In what follows, we shall demonstrate the effectiveness of EIS by comparing its final estimates of value function with the optimal one obtained via standard value iteration for games.

\medskip
\noindent{\bf Setup.} Consider a two-player turn-based Markov game $ (\mS_1, \mS_2,\mA_1, \mA_2, r, P, \gamma), $ where $\mS_1=[0.1,1.1]$ and $\mS_2=[-1.1,-0.1]$ are the set of states controlled by $\pone$ and $\ptwo,$ respectively; $\mA_1=\mA_2=\{0.1,0.2,0.3,0.4,0.5\}$ are the set of actions; $r(s,a)=3(|s|-0.5)^2-a$
is the reward received by $\pone$ when the corresponding player $I(s)$ takes action $a$ at state $s$. For each real number $u,$ define two clipping operators $\Pi_{\mS_{i}}(u)=\min\{\max\{\min{\mS_{i}},u\},\max{\mS_{i}}\},$ $i\in \{1,2\}.$ That is, $\Pi_{\mS_i}(u)$ projects $u$ to the state space of player $i.$ At state $s,$, upon taking an action $a$ by the corresponding player $I(s),$ the system transitions to next state $s'\sim \Pi_{\mS_{i}}\big(-|s|+\mathcal{N}(a,1)\big),$ where $\mathcal{N}(\cdot,\cdot)$ is the Guassian distribution, and $i=1$ if $I(s)=2,$ and $i=2$ if $I(s)=1.$ We consider the case $\gamma=0.8.$

Before proceeding, let us intuitively understand how the solution of this game should be. Recall that $\pone$ is the referenced reward maximizer. Since the reward $r(s,a)=3(|s|-0.5)^2-a$, by design, $\pone$ would prefer to stay at states that are far from $0.5$. That is, we expect the value function for $\pone$ in the range $[0.1,1.1]$ (recall that $\mS_1=[0.1,1.1]$) to be larger at states far from 0.5. On the contrary, as the reward minimizer, $\ptwo$ would like to stay at state $-0.5$ (recall that $\mS_2=[-1.1,-0.1]$) and potentially take large action $a$ to minimize the reward. As such, we expect the value function in $[-1.1, 0.1]$ to be small around state $-0.5$. We will confirm this intuition with both value iteration and our EIS method. 

\medskip
\noindent{\bf Algorithm.} We apply the EIS algorithm to learn the Nash equilibrium of the Markov game described above. Specifically, we use nearest neighbor regression as the supervised learning module and random sampling as the exploration module (cf. Section~\ref{sec:example})). As the Markov game is stochastic, we will use a variant of sparse sampling method~\cite{kearns2002sparse} as the improvement module. This algorithm is simple to describe and analyze, while suffices to convey essential insights.

\medskip
\noindent{\bf{Sparse Sampling Oracle.}} The sparse sampling algorithm can be viewed as a form of non-adaptive tree search for estimating the value of a given state. In particular, each node on the tree represents a state and each edge is labeled by an action and a reward. Starting from the root node, i.e., the queried state $s_0$ of the oracle, the tree is built in a simple manner: consider a node (state) $s$, for each action $a\in\mA$, call the generative model $C$ times on the state-action pair $(s,a)$ and  obtain $C$ children of the action $a$; the children nodes are the states returned by the generative model, and each edge is labeled by the action $a$ and the reward returned by the generative model. The process is repeated for all nodes of each level, and then moves on to the next level until reaching a depth of $H$. In essence, this process builds a $|\mA|C$-array tree of depth $H$. It represents a partial $H$-step look-ahead tree starting from the queried state $s_0$, and hence the term sparse sampling oracle. 

To obtain estimates for the value of the queried state, estimation from the supervised learning module are assigned to the leaf nodes at depth $H$. These values, together with the associated rewards on the edges, are then backed-up to find estimates of values for their parents, i.e., nodes at depth $H-1$.  
The backup is just a simple average over the children, followed by taking appropriate max or min operation depending on who is the acting player at this layer of the tree. The process is recursively applied from the leaves up to the root level to find estimates of $\hat{V}(s_0)$ for the root node $s_0$. 
In the experiments, we set $H=2,C=30.$

\begin{figure}[htp]
\centering
\includegraphics[width=3.2in]{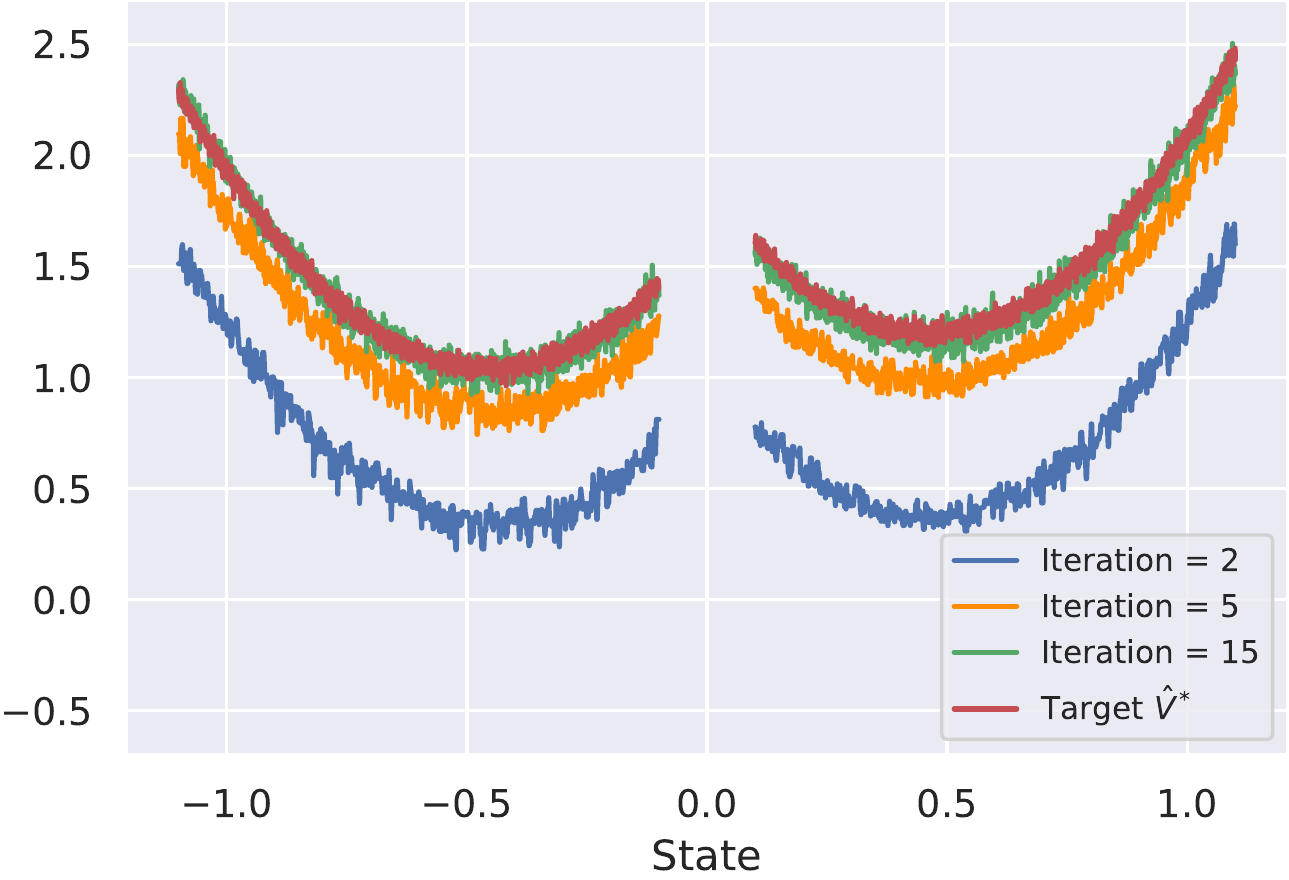}
\includegraphics[width=3.12in]{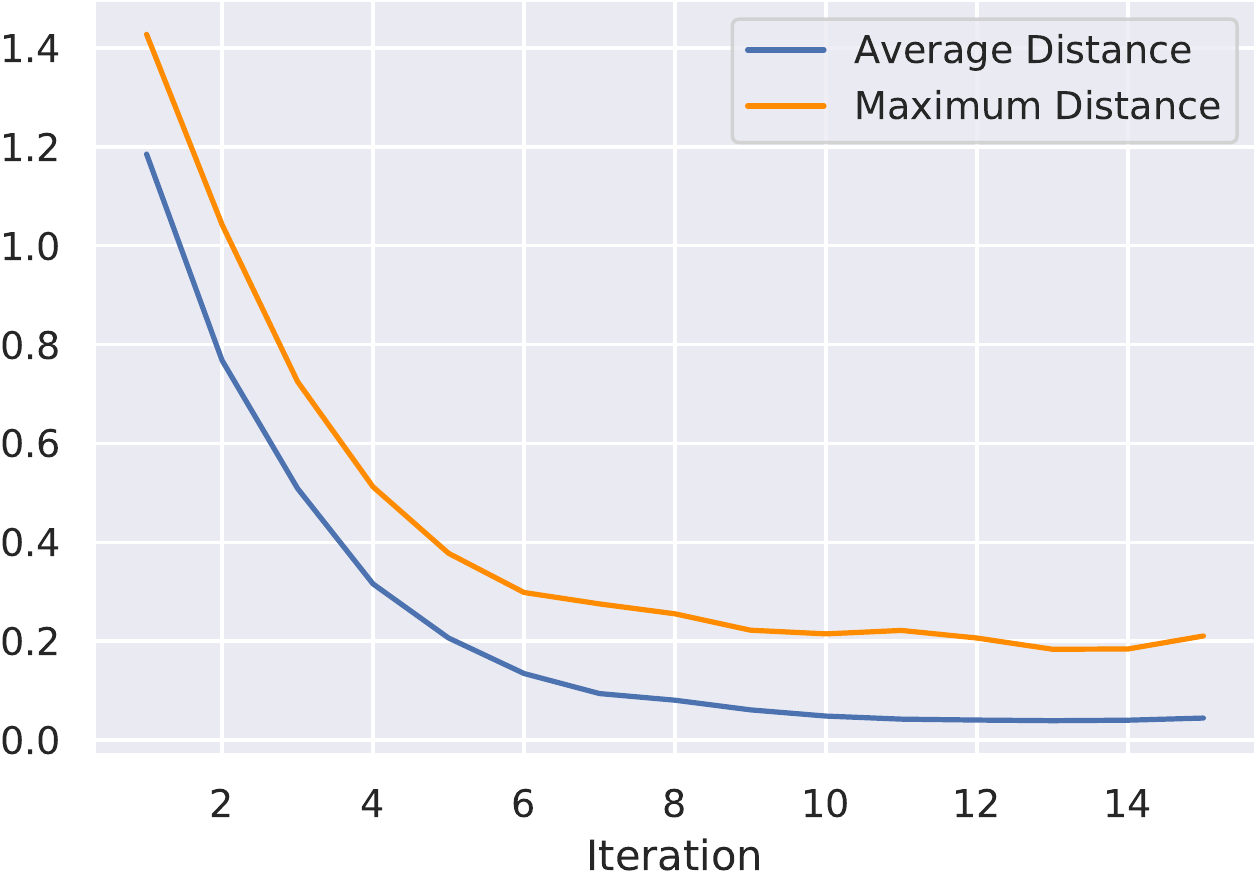}
\caption{Results of EIS for various iterations. {Left:} Approximate optimal $\hat{V}^*$ and the value function estimation $V_t$ of EIS obtained at various iterations. {Right:} Average distance and maximum distance between $\hat{V}^*$ and $V_t$ at each EIS iteration. }
\label{fig:eis_result}

\end{figure}

\medskip
\noindent{\bf Evaluation.} We first use approximate value iteration to compute the optimal value function $V^*$. The value iteration operator $\mathcal{T}$ is defined as follows:  
\[
(\mathcal{T}V)(s)=\begin{cases}
\max_{a\in\mA_{1}}\big\{ r(s,a)+\gamma\E[V(s')|s_0=s,a_0=a]\big\}, & \text{if }s\in\mS_{1}\\
\min_{a\in\mA_{2}}\big\{ r(s,a)+\gamma\E[V(s')|s_0=s,a_0=a]\big\}, & \text{if }s\in\mS_{2}
\end{cases}
\]
For the continuous game considered here, we discretize the state space of each player to be $1,500$ equally spaced states. The above value iteration operator is then applied to obtain an approximate optimal value function ${V}^*$. 
The approximate ${V}^*$ generated by $30$ iterations of value iteration is plotted in red in Figure~\ref{fig:eis_result}~(Left).
As expected, the result is consistent with our intuition that $\pone$ attempts to stay away from $0.5$ and $\ptwo$ tries to minimize the reward by staying around $-0.5$. 

Next, we evaluate our EIS method and compare the outputs against the approximate $\hat{V}^*$. In particular, let $V_t$ denote the value function obtained by EIS after $t$ iterations. Figure~\ref{fig:eis_result} (Left) shows the progress of EIS (i.e., $V_t$) at various iterations. It is clear that EIS gradually improves the estimation of value function. On the right of Figure~\ref{fig:eis_result}, we plot the average distance as well as the maximum distance between $V_t$ and $\hat{V}^*$ over $15$ iterations. We remark that there is an inevitable gap between $\hat{V}^*$ and $V^*$ due to discretization. As can be seen, the error of EIS output decays gradually.